\documentclass[letterpaper, 10 pt, conference]{ieeeconf} 

\IEEEoverridecommandlockouts                              

\usepackage{fancyhdr}

\usepackage{xcolor}
\usepackage{comment}

\usepackage{cite}

\usepackage[mathcal]{euscript}

\usepackage[cmex10]{amsmath}
\usepackage{amssymb}

\usepackage{amsthm} 
  
\usepackage{dsfont} 

\usepackage{epsfig} 

\usepackage{multirow}
\usepackage{enumerate}
\usepackage{array}

\let\labelindent\relax
\usepackage{enumitem}

\usepackage{diagbox} 
\usepackage{adjustbox} 

\usepackage[ruled, vlined, linesnumbered]{algorithm2e}
\SetAlCapFnt{\small}
\SetNlSty{}{}{} 
\SetInd{0.5em}{0.2em}
\SetAlgoSkip{9.5pt}

\newtheoremstyle{plain}
	  {}
	  {}
	  {\itshape}
	  {}
	  {\bfseries}
	  {}
	  {5pt plus 1pt minus 1pt}
	  {}

\newtheoremstyle{definition}
  	  {}
	  {}
	  {\normalfont}
	  {}
	  {\bfseries}
	  {}
	  {5pt plus 1pt minus 1pt}
	  {}
  
\theoremstyle{plain}
\newtheorem{assumption}{Assumption}

\newtheorem{lemma}{Lemma}
\newtheorem{proposition}{Proposition}

\theoremstyle{definition}

\newtheorem{remark}{Remark}

\newcommand{\refeq}[1]			{(\ref{#1})} 
\newcommand{\reffig}[1]			{Fig. \ref{#1}} 
\newcommand{\refsec}[1]			{Section \ref{#1}}

\newcommand{\reftab}[1]			{Table \ref{#1}}

\newcommand{\refprop}[1]		{Proposition \ref{#1}}

\newcommand{\reflem}[1]			{Lemma \ref{#1}}

\newcommand{\refasm}[1]			{Assumption \ref{#1}}
\newcommand{\refalg}[1]			{Algorithm \ref{#1}}

\newcommand{\refrem}[1]			{Remark \ref{#1}}

\theoremstyle{plain}

\newcommand{\R}  	{\mathbb{R}} 
\newcommand{\N}     {\mathbb{N}} 

\newcommand{\radius} 	{\rho}

\newcommand{\conv}      {\mathrm{conv}} 

\newcommand{\ovect}[1] {\begin{bmatrix}\cos #1\\ \sin #1 \end{bmatrix}}
\newcommand{\ovectsmall}[1] {\scalebox{0.85}{$\ovect{#1}$}}

\newcommand{\ovecT}[1] {\tr{\ovect{#1}\!}}

\newcommand{\ovecTsmalll}[1] {\scalebox{0.65}{$\ovecT{#1}$}}
\newcommand{\nvect}[1] {\begin{bmatrix}-\sin #1\\ \cos #1 \end{bmatrix}}

\newcommand{\nvecT}[1] {\tr{\nvect{#1}\!}}

\newcommand{\nvecTsmalll}[1] {\scalebox{0.65}{$\nvecT{#1}$}}

\newcommand{\motionset}{\mathcal{MP}} 

\newcommand{\workspace}	{\mathcal{W}} 
\newcommand{\map}     {\vect{M}} 
\newcommand{\occupiedspace}	{\mathcal{O}} 
\newcommand{\unknownspace}	{\mathcal{U}} 
\newcommand{\occupiedspaceunknown} {\occupiedspace_{\mathrm{real}}}
\newcommand{\freespaceunknown} {\freespace_{\!\mathrm{real}}}

\newcommand{\mappct}{\mathrm{MapPct}}

\newcommand{\clearance} {\varepsilon}

\newcommand{\freespace}	{\mathcal{F}} 

\newcommand{\planningfreespace} {\freespace_{\!\mathrm{plan}}}

\newcommand{\controlfreespace} {\freespace_{\!\mathrm{ctrl}}}

\newcommand{\erode} {\mathrm{erode}}

\newcommand{\frontierspace} {\mathcal{FR}} 
\newcommand{\frontierregion} {\mathcal{FR}} 
\newcommand{\frontierregionFunc} {\mathrm{frontierRegions}} 
\newcommand{\frontier} {\mathrm{f}} 

\newcommand{\info} {\mathrm{info}} 
\newcommand{\navcost} {\mathrm{navcost}} 
\newcommand{\mininfo}{\mu}

\newcommand{\probability} {\mathrm{p}} 
\newcommand{\freeprobability} {\probability_{\mathrm{free}}}
\newcommand{\occupiedprobability} {\probability_{\mathrm{occ}}}
\newcommand{\entropy} {\mathrm{entropy}} 
\newcommand{\volume}{\mathrm{volume}} 

\newcommand{\visitcost} {\mathrm{visitcost}}

\newcommand{\explorationPlan}{\mathrm{explorationPlan}}

\newcommand{\dist}      {\mathrm{dist}}

\newcommand{\distcoll} {\mathrm{dist2coll}} 
\newcommand{\distunknown}{\mathrm{dist2unkn}}

\newcommand{\ball}      {\mathrm{B}} 

\newcommand{\path}      {\mathrm{p}}

\newcommand{\travelcost} {\mathrm{travelcost}}
\newcommand{\optimalpath}      {\mathrm{optpath}}

\newcommand{\pathparam} {s}

\newcommand{\smin}      {0}
\newcommand{\smax}      {1}

\newcommand{\getmap} {\mathrm{getCurrentMap}}
\newcommand{\getState} {\mathrm{getCurrentState}}

\newcommand{\isinformative} {\mathrm{isInformative}}
\newcommand{\iscomplete} {\mathrm{isComplete}}

\newcommand{\isnear}{\mathrm{isNear}}

\newcommand{\treplan} {\Delta t}

\newcommand{\pos} 		{\vect{x}} 			
\newcommand{\ort}	    {\theta}			

\newcommand{\linvel}     {v}  
\newcommand{\angvel}     {\omega} 

\newcommand{\lingain}   	{\kappa_{v}} 			
\newcommand{\anggain}   	{\kappa_{\omega}} 			

\newcommand{\goal}		{\vect{x}^*} 

\newcommand{\gain}   	{\kappa} 			
\newcommand{\ctrl}      {\vect{u}} 			

\newcommand{\sensorrange}{r_{\mathrm{max}}} 

\newcommand{\safelevel}{\sigma} 

\newcommand{\viewpointset} {\mathrm{V}}
\newcommand{\viewpoint} {\vect{v}}

\newcommand{\visiblefrontiers}{\Lambda}
\newcommand{\visibletolerance}{\eta}

\newcommand{\view} {\mathrm{view}}

\let\originalleft\left
\let\originalright\right
\renewcommand{\left}{\mathopen{}\mathclose\bgroup\originalleft}
\renewcommand{\right}{\aftergroup\egroup\originalright}
	
\newcommand{\plist}[1] 	{\left(#1\right)} 
\newcommand{\blist}[1]	{\left[ #1 \right]} 
\newcommand{\clist}[1]	{\left\{#1\right\}} 

\DeclareMathOperator{\atantwo}{atan2}

\newcommand{\vect}[1]   {\mathrm{#1}}

\newcommand{\tr}[1] {{#1}^{\mathrm{T}}} 
\newcommand{\norm}[1]  {\|#1\|}
\newcommand{\absval}[1]{\left|#1 \right|} 

\newcommand{\ldf}   {:=} 

\newcommand{\argmin}{\operatornamewithlimits{arg\ min}} 
\newcommand{\argmax}{\operatornamewithlimits{arg\ max}} 
\newcommand{\diff} {\mathrm{d}}

\makeatletter
\DeclareRobustCommand\widecheck[1]{{\mathpalette\@widecheck{#1}}}
\def\@widecheck#1#2{%
    \setbox\z@\hbox{\m@th$#1#2$}%
    \setbox\tw@\hbox{\m@th$#1%
       \widehat{%
          \vrule\@width\z@\@height\ht\z@
          \vrule\@height\z@\@width\wd\z@}$}%
    \dp\tw@-\ht\z@
    \@tempdima\ht\z@ \advance\@tempdima2\ht\tw@ \divide\@tempdima\thr@@
    \setbox\tw@\hbox{%
       \raise\@tempdima\hbox{\scalebox{1}[-1]{\lower\@tempdima\box
\tw@}}}%
    {\ooalign{\box\tw@ \cr \box\z@}}}
\makeatother

\title{\LARGE \bf
Action-Aware Pro-Active Safe Exploration for Mobile Robot Mapping
}

\author{Aykut \.{I}\c{s}leyen and René van de Molengraft and \"{O}m\"{u}r Arslan
\thanks{The authors are with the Department of Mechanical Engineering, Eindhoven University of Technology, P.O. Box 513, 5600 MB Eindhoven, The Netherlands.
The authors are also affiliated with the Eindhoven AI Systems Institute.
Emails:  \{a.isleyen, m.j.g.v.d.molengraft, o.arslan\}@tue.nl}%
}

\begin{document}

\maketitle
\thispagestyle{empty}
\pagestyle{empty}

\begin{abstract}

Safe autonomous exploration of unknown environments is an essential skill for mobile robots to effectively and adaptively perform environmental mapping for diverse critical tasks. 
Due to its simplicity, most existing exploration methods rely on the standard frontier-based exploration strategy, which directs a robot to the boundary between the known safe and the unknown unexplored spaces to acquire new information about the environment.
This typically follows a recurrent persistent planning strategy, first selecting an informative frontier viewpoint, then moving the robot toward the selected viewpoint until reaching it, and repeating these steps until termination.
However, exploration with persistent planning may lack adaptivity to continuously updated maps, whereas highly adaptive exploration with online planning often suffers from high computational costs and potential issues with livelocks.
In this paper, as an alternative to less-adaptive persistent planning and costly online planning, we introduce a new proactive preventive replanning strategy for effective exploration using the immediately available actionable information at a viewpoint to avoid redundant, uninformative last-mile exploration motion.
We also use the actionable information of a viewpoint as a systematic termination criterion for exploration.
To close the gap between perception and action, we perform safe and informative path planning that minimizes the risk of collision with detected obstacles and the distance to unexplored regions, and we apply action-aware viewpoint selection with maximal information utility per total navigation cost. 
We demonstrate the effectiveness of our action-aware proactive exploration method in numerical simulations and hardware experiments. 
\end{abstract}

\section{Introduction}
\label{sec.Introduction}

Safe exploration of a priori unknown environments is an essential skill for autonomous mobile robots to adaptively and effectively map their surroundings \cite{stachniss_RoboticMappingExploration2009} in order to perform critical tasks such as search and rescue \cite{calisi_etal_ISRR2005}, surveillance \cite{dipaola_IJARS2010}, environmental monitoring \cite{ghaffari_etal_IJRR2019}, and planetary exploration \cite{otsu_agah-mohammadi_paton_RAL2018}.
Due to its simplicity and effectiveness in practice, frontier-based exploration \cite{yamauchi_CIRA1997} is a widely adopted approach for autonomous robotic exploration and mapping, guiding a mobile robot toward frontier locations at the boundary between known safe and unknown unexplored regions to gather new information about the environment.
The frontier-based exploration typically follows a recurrent persistent planning strategy, first selecting an informative next-best frontier viewpoint based on a weighted combination of information-theoretic and travel-distance heuristics, then safely moving the robot toward the selected viewpoint until reaching it, and repeating these exploration steps until a termination condition (e.g., maximum exploration time or distance, or minimum frontier size) is met \cite{placed_etal_TRO2023,lluvia_etal_Sensors2021}.
However, exploration with persistent planning may lack adaptivity to continuously updated maps during robot motion, whereas highly adaptive exploration with online planning often suffers from high computational costs and potential issues with livelocks where the robot continuously switches between different exploration goals without making progress.
Moreover, although robotic exploration is inherently an integrated perception and action task, most existing approaches primarily treat it as a perception-driven high-level decision-making problem, overlooking the importance of low-level planning and control for safe exploration and execution in unknown environments~\cite{placed_etal_TRO2023}.

\begin{figure}[t]
    \centering
        \includegraphics[width = 0.9835\columnwidth]{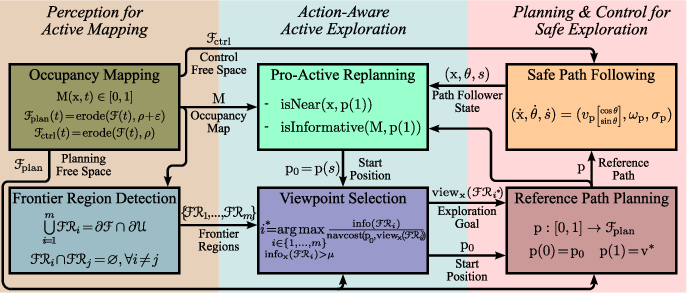} \\[0.5mm]
    \begin{tabular}{@{}c@{\hspace{0.005\columnwidth}}c@{\hspace{0.005\columnwidth}}c@{}}
        \includegraphics[width = 0.38\columnwidth]{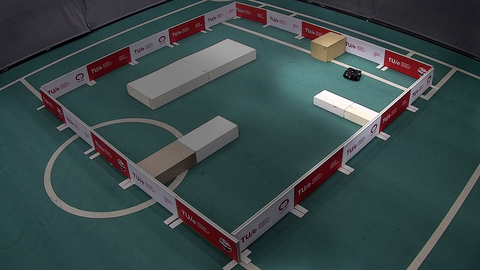} &
        \includegraphics[width = 0.38\columnwidth]{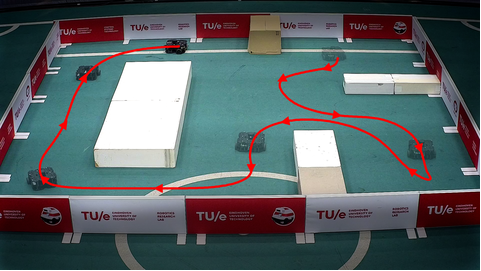} &
        \includegraphics[width = 0.2135\columnwidth]{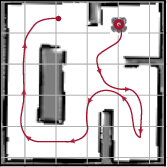} 
    \end{tabular}
    \vspace{-2mm}
\caption{
Action-aware proactive exploration incrementally maps an unknown environment by using safe and informative paths toward the best frontier viewpoints maximizing information utility per total navigation cost, along with systematic preventive replanning based on the immediately available actionable information of viewpoints to avoid redundant last-mile motion. (top) The elements of the action-aware proactive exploration framework. (bottom) An example of safe autonomous robotic exploration with the resulting occupancy map and the robot's exploration trajectory.
  }
    \label{fig.framework}
    \vspace{-\baselineskip}
    \end{figure}

In this paper, to bridge the gap between perception and action in robotic exploration, we propose a verifiably complete and safe  action-aware proactive exploration framework that systematically combines safe and informative path planning, action-aware informative viewpoint selection, proactive replanning, and safe path-following control, as illustrated in \reffig{fig.framework}.
In particular, we introduce a new proactive preventive replanning strategy for effective exploration, leveraging the immediately available actionable information at a viewpoint to avoid potential redundant, uninformative last-mile motion at the end of an exploration plan.
For safe and informative motion planning, we construct optimal maximal clearance informative paths that minimize the risk of collision with detected obstacles and the distance to unexplored regions along the path.
For action-aware frontier viewpoint selection, we choose a safely reachable frontier viewpoint candidate with maximal information content per total navigation cost.

\subsection{Motivation and Relevant Literature}

Robotic exploration for active mapping consists of three major elements: perception for exploration (e.g., occupancy mapping and frontier detection),  viewpoint selection for exploration (e.g., the information utility and navigation cost of frontier viewpoints), and  safe and informative planning and control for exploration, which are mostly performed in a sequential manner but may also be combined together\,\mbox{\cite{placed_etal_TRO2023,lluvia_etal_Sensors2021}.}

\subsubsection{Perception for Exploration}

Autonomous robotic exploration for active mapping aims to leverage the motion of a mobile robot with onboard sensing (e.g., camera or LiDAR) to incrementally build a globally consistent and accurate representation of a priori unknown spatial information about the environment (e.g., occupancy maps with obstacles) \cite{placed_etal_TRO2023, lluvia_etal_Sensors2021}.
Perception for exploration mainly focuses on map representation and viewpoint detection.
Due to their simplicity, occupancy grid maps are a widely used probabilistic metric world representation, serving as a fine tessellation of environments into simple grid shapes, and also allowing for easy classification of the space into obstacle-free, obstacle-occupied, and unknown regions via simple probability thresholding \cite{elfes_Computer1989, thrun_AR2003, thrun_probabilistic2005}.
Over such a metric world representation, all locations are potential viewpoints to gather new information about the environment to increase mapping quality (e.g., by decreasing map uncertainty),
however, this often results in an intractable optimization and search problem for determining the best robot motion and action. \cite{stachniss_RoboticMappingExploration2009}.  
This computational challenge is often mitigated by focusing the search on the more promising locations, known as frontiers, at the boundary of the known obstacle-free and unknown regions \cite{yamauchi_CIRA1997}.
Over an occupancy grid map, frontier locations can be efficiently detected using standard image processing techniques (e.g., edge detection) and maintained incrementally \cite{keidar_kaminka_IJRR2014}. 
Randomized frontier detection is also applied to effectively handle high-dimensional and complex environments \cite{umari_mukhopadhyay_IROS2017}, and the notion of frontiers naturally extends to hybrid topological-metric maps, built as a pose graph of submaps \cite{grisetti_etal_MITS2010, latha_arslan_arXiv2024}, as well as to semantic maps \cite{gomez_etal_Sensors2019, kim_RAL2022}.
To further reduce the search space, frontiers are usually clustered into coherent, connected regions and assigned a representative point (e.g., centroid) using connected-component analysis \cite{yamauchi_CIRA1997}, k-means clustering \cite{lu_redondo_campoy_Sensors2020}, or mean-shift clustering \cite{umari_mukhopadhyay_IROS2017}.
A commonly known challenge with the representative center of a frontier region is that it may be unsafe to reach (e.g., in unknown or occupied space) or safe but unreachable \cite{uslu_INISTA2015}.
In this paper, we limit the search for the best viewpoint selection to safely reachable viewpoints of frontier regions  that allow a clear and reliable visibility of frontiers within the robot's sensing range.

\subsubsection{Viewpoint Selection for Exploration}

The search for the best viewpoint selection begins by assessing the quality of potential viewpoint candidates based on perception-driven information-theoretic utility measures and action-driven navigation costs \cite{placed_etal_TRO2023}.
The perception-driven information utility of a viewpoint is often measured by the entropy of the surrounding space on the map \cite{bourgault_etal_IROS2002, stachniss_grisetti_burgard_RSS2005, charrow_etal_RSS2015} or the size of the nearby unknown space \cite{connolly_ICRA1985, gonzalez-banos_latombe_IJRR2002}. 
Learning-based prediction of the information utility of viewpoints has also been applied to support robotic exploration through deep learning \cite{shrestha_etal_ICRA2019, saroya_best_hollinger_IROS2020}, reinforcement learning \cite{niroui_etal_RAL2019}, and Gaussian processes \cite{bai_etal_IROS2016}.
Meanwhile, the action-driven navigation cost of a viewpoint relative to the robot's position is usually determined based on the standard Euclidean distance \cite{yamauchi_CIRA1997}, shortest path (i.e., geodesic) distance \cite{julia_gil_reinoso_AR2012}, travel time \cite{zhou_etal_RAL2021}, energy consumption \cite{mei_lu_lee_hu_ICRA2006}, and orientational alignment \cite{cieslewski_etal_IROS2017}.
As expected, low-quality, simplistic viewpoint selection, based solely on perception-driven information utility or straight-line Euclidean navigation cost, often results in inefficient exploration motion, with significant back-and-forth movement between different viewpoints \cite{julia_gil_reinoso_AR2012}. 
This is because purely information-driven viewpoint selection ignores the required robot motion, as if the robot teleports, whereas simplistic action models (e.g., the straight-line paths) overlook the influence of obstacles on robot motion, creating a significant disparity between high-level decision-making and low-level planning and control, which results in poor performance during the execution of exploration.
Hence, additive, multiplicative, and normalized combinations of information utility and navigation cost are often preferred to improve exploration efficiency \cite{gonzalez-banos_latombe_IJRR2002, stachniss_RoboticMappingExploration2009, dai_etal_ICRA2020}, but this often requires careful parameter and weight tuning \cite{placed_etal_TRO2023, gomez_etal_Sensors2019}.
In this paper, we consider a new action-aware viewpoint selection strategy that maximizes the information utility per total navigation cost, where the navigation cost is determined along safe and informative paths that minimize the risk of collision with detected obstacles and the distance to unexplored space.

\subsubsection{Planning and Control for Exploration}

Once the best viewpoint for exploration is selected as a high-level exploration goal, the safe and successful execution of exploration requires a systematic low-level planning and control approach to avoid collisions with unseen obstacles, and prevent deadlocks (where the robot becomes stuck for extended periods) and livelocks (where the robot continuously switches between viewpoints without making progress).
To avoid livelocks at the cost of reduced adaptability to continuously updated maps, autonomous exploration usually adopts a persistent planning strategy to safely move the robot toward the selected viewpoint until it is reached, and then exploration continues with the next best viewpoint \cite{placed_etal_TRO2023, lluvia_etal_Sensors2021} until a termination condition (e.g., maximum exploration time or distance, or minimum frontier size) is met \cite{placed_IFAC2022}.
Many existing studies delegate the execution of exploration with persistent planning to off-the-shelf software packages (e.g., the ROS navigation stack) \cite{gao_etal_ICARCV2018, gomez_etal_Sensors2019, umari_mukhopadhyay_IROS2017, strom_RAS2017, kim_RAL2022}, but these tools require careful and complex parameter tuning for control, local and global planning, and recovery behaviors to ensure proper and reliable robot operation \cite{zheng_Spring2021}.
As a result, to avoid deadlocks due to the discrepancy between exploration planning and navigation planning, they often necessitate frequent, computationally costly replanning for navigation in unknown environments \cite{koenig_TRO2005, tordesillas_TRO2021}.
A potential solution to robust and reliable exploration execution is exploration with online planning based on integrated receding-horizon planning and control  \cite{bircher_ICRA2016, dharmadhikari_etal_ICRA2020, lindqvist_etal_TRO2024, hollinger_sukhatme_IJRR2014}.
However, such integrated optimal planning and control approaches inherit both adaptability and reliability, as well as the high computational cost and potential issues with local minima and livelocks in online planning \cite{lluvia_etal_Sensors2021}, and are often combined with global persistent planning to overcome these issues \cite{charrow_etal_RSS2015}.
Hence, to take advantage of the computational benefits of persistent planning with safe and reliable execution of exploration, we adopt a verifiably correct and safe path-following control approach using feedback motion prediction and online path time parameterization for nonholonomically constrained unicycle mobile robots \cite{arslan_arXiv2022, isleyen_vandewouw_arslan_IROS2023, isleyen_vandewouw_arslan_cdc2023}.
Last but not least, the termination of an exploration plan or the entire exploration process is an important decision in autonomous exploration. An exploration plan is often terminated if it is not safely executable due to an unreachable viewpoint selection \cite{yamauchi_CIRA1997} or if the viewpoint has no frontier neighbor \cite{holz_etal_ISR2010}, whereas the entire exploration is stopped based on the maximum allowed exploration time or distance, minimum frontier size, or map entropy \cite{placed_IFAC2022}.
In this paper,  we introduce a new notion of immediately available actionable information of a viewpoint as a systematic termination criterion for exploration and we propose a proactive, preventive replanning strategy for effective exploration based on the actionable information of a viewpoint to avoid redundant, uninformative last-mile exploration motions around it, see \reftab{tab.ExplorationStrategies}.

\begin{table}
    \caption{Planning Strategies for Active Exploration}
    \label{tab.ExplorationStrategies}
    \vspace{-3mm}
    \centering
    \resizebox{\columnwidth}{!}{
    \begin{tabular}{@{}c@{\hspace{1mm}}c@{}}
    \hline 
    \hline
    \\[-2mm]
    \multicolumn{1}{@{}c@{}}{\textbf{Exploration Strategy}} & \multicolumn{1}{c}{\textbf{Exploration Procedure}}                                                                         \\[1mm] 
    \hline
    \\[-2mm]
    \parbox{0.36\columnwidth}{\centering Persistent Planning}
    &
    \parbox{0.75\columnwidth}{Plan an exploration path toward a selected viewpoint and persistently follow it until the viewpoint is reached.}
    \\[2mm]
    \hline
    \\[-2mm]
    \parbox{0.36\columnwidth}{\centering Online Planning}
    &
    \parbox{0.75\columnwidth}{Continuously (re)plan and follow an exploration path toward the current best viewpoint.}
    \\[2mm]
    \hline
    \\[-2mm]
    \parbox{0.36\columnwidth}{\centering Last-Mile\\ 
    Preventive Planning}
    &
    \parbox{0.75\columnwidth}{Plan an exploration path toward a selected viewpoint and follow it as long as the viewpoint remains informative; otherwise, replan an exploration path to a new viewpoint.}
    \\[3.5mm]
    \hline \hline
    \end{tabular}
    }
    \vspace{-1.5\baselineskip}
\end{table}
\subsection{Contributions and Organization of the Paper}

This paper presents an action-aware proactive exploration framework for safe and effective occupancy mapping of unknown environments by tightly and systematically coupling perception and action in three key steps of exploration, including safe and informative path planning, action-aware viewpoint selection, and actionable-information-based preventive replanning. 
In summary, our main contributions are:

\begin{itemize}

\item We describe how to construct an optimally safe and informative exploration plan that minimizes the risk of collision with obstacles as well as the distance to unexplored regions.

\item We propose an action-aware, safely reachable, and informative viewpoint selection strategy that maximizes information utility per total navigation cost.

\item We introduce a new notion of the immediately available actionable information of a viewpoint to terminate an exploration plan as a preventive replanning strategy and also to determine the completion of exploration. 
    
\item We describe how to safely execute an exploration plan without deadlocks and perform online replanning without livelocks by using safe and persistent path-following control for nonholonomic mobile robots.
 
\end{itemize}  

\noindent We demonstrate the effectiveness of our action-aware proactive exploration framework for mobile mapping in numerical ROS-Gazebo simulations and real physical hardware experiments by investigating the influence of information utility, navigation cost, and proactive replanning on exploration performance.
We observe that action-aware exploration with proactive replanning, taking into account the cost of navigation, significantly improves the rate and speed of exploration compared to information-only exploration strategies.
Because all informative frontier regions, irrespective of their information content, are eventually required to be observed to complete the mapping, making the cost of navigation more influential in robotic exploration than the information utility.



The rest of of the paper is organized as follows.  
\refsec{sec.perception_for_active_mapping} briefly presents the standard occupancy mapping and frontier region detection methods as the perception elements of active mapping.
\refsec{sec.planning_control_for_safe_exploration} describes how to optimally plan a safe and informative reference path for exploration, and how to safely and reliably execute it with a nonholonomic mobile robot.
\refsec{sec.action_aware_active_exploration} present our active-aware frontier viewpoint selection and pro-active last-mile-preventive replanning strategy for effective and fast exploration and mapping in unknown environments.
\refsec{sec.numerical_simulations_physical_experiments} demonstrates the influence of information utility and navigation cost on exploration, both in numerical simulations and physical hardware experiments.
\refsec{sec.conclusions} concludes with a summary of our work and results, as well as future research directions.

\section{Perception for Active Mapping}
\label{sec.perception_for_active_mapping}

In this section, we briefly present two widely used perception elements of active exploration for autonomous mapping: occupancy grid mapping as a metric world representation and frontier region detection as exploration target~selection.

\subsection{Occupancy Mapping}

For ease of exposition, we consider a disk-shaped mobile robot with a body radius $\radius \in \R_{>0}$, centered at position $\pos \in \workspace$, moving in an unknown two-dimensional static bounded workspace $\workspace \subset \R^2$, consisting of a priori unknown obstacles $\occupiedspaceunknown \subset \workspace$ and an unknown obstacle-free space $\freespaceunknown \ldf \workspace \setminus \occupiedspaceunknown$.
We assume that the robot's position $\pos(t)$ is known at all times $t \geq 0$ with perfect localization, and the robot is equipped with a 2D (e.g., LiDAR) range scanner centered at the robot's position $\pos$, with a fixed maximum sensing range $\sensorrange > \radius$ (greater than the robot's body radius $\radius$) and a $360^\circ$ angular scanning range.
Accordingly, we consider the availability of a probabilistic occupancy (grid) mapping method \cite{thrun_probabilistic2005} that uses the known robot's position and scan readings to continuously update a probabilistic occupancy map, denoted by $\map(\pos,t) \in \blist{0,1}$, of the workspace $\workspace$, using Bayesian filtering and ray casting, where $\map(\pos,t) \in [0,1]$ represents the occupancy likelihood of each position $\pos \in \workspace$ at time $t \geq 0$.
Based on the occupancy map probabilities $\map(\pos,t)$, the workspace $\workspace$ can be partitioned into three categorical regions: obstacle-free $\freespace(t)$, obstacle-occupied $\occupiedspace(t)$, and unknown $\unknownspace(t)$ spaces as
\begin{subequations}\label{eq.free_occupied_unknown_spaces}
\begin{align}
    &\freespace(t) \ldf \clist{ \pos \!\in\! \workspace \mid \map(\pos,t) \leq \freeprobability }, \label{eq.freespace}\\
    &\occupiedspace(t) \ldf \clist{ \pos \!\in\! \workspace \mid \map(\pos,t)  \geq \occupiedprobability }, \label{eq.occupiedspace}\\
    &\unknownspace(t) \ldf \clist{ \pos \!\in\! \workspace \mid \freeprobability < \map(\pos,t)  < \occupiedprobability }, \label{eq.unknownspace}
\end{align}
\end{subequations}
where $\freeprobability$ and $\occupiedprobability$, with $0 \leq \freeprobability \!<\! \occupiedprobability \leq 1$, are  fixed probability thresholds, describing whether a position $\pos \in \workspace$ is free or occupied, respectively.
Moreover, we use $\map(t)$ to refer to the complete occupancy map at time $t$.

We assume that the probabilistic occupancy map $\map(t)$ of the environment is updated in a consistent manner. 

\begin{assumption} \label{asm.consistent_occupancy_mapping} (Consistent Occupancy Mapping)
For any time $t \geq 0$, regardless of how the robot position $\pos(t)$ changes over time,  we assume the occupancy mapping satisfies that
\begin{itemize}
\item (Safe  Under-Approximation of Free Space) The continuously mapped free space $\freespace(t)$ is a subset and defines a safe under-approximation of the unknown actual free space $\freespaceunknown = \workspace \setminus \occupiedspaceunknown$ for all times, i.e.,
\begin{align}
\freespace(t) \!\subseteq\! \freespaceunknown  \quad \forall t \geq 0.
\end{align}
\item (Increasingly Inclusive Free \& Obstacle Spaces)~The continuously mapped obstacle-free and obstacle-occupied spaces, $\freespace(t)$ and $\occupiedspace(t)$, incrementally grow in an inclusive manner whereas the unknown space $\unknownspace(t)$ incrementally shrinks in an exclusive manner, i.e.,
\begin{align} \label{eq.freespace_expand}
\! \freespace(t') \!\supseteq\! \freespace(t), \!\quad\!
\occupiedspace(t') \!\supseteq\! \occupiedspace(t), \!\quad\!
\unknownspace(t') \!\subseteq\! \unknownspace(t) \!\quad\! \forall t' \!\geq t. \!\!\!
\end{align}
\item (Unknown Space Exploration) If there exists a visible unknown position $\vect{y} \in \unknownspace(t)$ within the robot's sensing range $\sensorrange$ from a safe robot position $\pos(t) \in \freespace(t)$, then the unknown space $\unknownspace(t)$ strictly shrinks over time, i.e.,
\begin{align}
\begin{array}{c}
\blist{\vect{y}, \pos(t)} \!\subset\! \workspace \setminus \occupiedspace(t), \\
\norm{\vect{y} \!-\! \pos(t)} \leq  \sensorrange
\end{array}
 \!\Rightarrow \unknownspace(t') \!\subset\! \unknownspace(t) \quad  \forall t' > t.\!\!
\end{align}
\end{itemize}
\end{assumption}

Under the assumption of accurate localization, e.g., using an external positioning system such as a motion capture system or global positioning system (GPS) or sensor-based onboard simultaneous localization and mapping methods \cite{thrun_probabilistic2005}, the consistent occupancy mapping requirement in \refasm{asm.consistent_occupancy_mapping} can be satisfied in many practical settings.
The close-loop nature of active mapping also reinforces the occupancy-consistent mapping assumption.
In this paper, we use a motion capture system in both simulations and experiments for accurate localization and apply \refasm{asm.consistent_occupancy_mapping} for the systematic and provably correct design of our action-aware safe mobile exploration framework.
Accordingly, for safe motion planning and control of a mobile robot with a finite body size, we separately construct the free spaces $\planningfreespace(t)$ and $\controlfreespace(t)$ for planning and control, by eroding the free space $\freespace(t)$ by the robot's body radius $\radius \!>\! 0$ and a fixed safe and smooth control tolerance $\clearance \!>\! 0$, as in \reffig{fig.over_under_approx}~as
\begin{align} 
\planningfreespace(t) &\ldf \erode(\freespace(t), \radius + \clearance), \label{eq.planningfreespace}
\\
\controlfreespace(t) &\ldf \erode(\freespace(t), \radius), \label{eq.controlfreespace}
\end{align}
where $\clearance $ is an additional safety clearance that facilitates continuous maneuvering and control around obstacles with (e.g., nonholonomic) motion constraints \cite{isleyen_vandewouw_arslan_IROS2023}, and the erosion of a set $\mathcal{A}$ by a circular ball of a radius $\radius\!>\!0$ is defined as
\begin{align} \label{eq.erosion}
    \erode(\mathcal{A}, \radius) \ldf \clist{ \pos \in \mathcal{A} \mid \ball(\pos, \radius) \subseteq \mathcal{A} },
\end{align}
where $\ball(\pos, \radius) \ldf \clist{ \vect{y} \in \R^2 \mid \norm{\vect{y} - \pos} \leq \radius }$ is the Euclidean closed ball centered at $\pos$ with radius $\radius\!\geq\!0$ and $\norm{\cdot}$ denotes the standard Euclidean norm.

\begin{remark}\label{rem.safe_planning_control_spaces}
\emph{(Safe Planing and Control Spaces)}
The safe planning space $\planningfreespace(t)$ is a subset of the safe control space $\controlfreespace(t)$, which in turn is a subset of the obstacle-free space $\freespace(t)$, i.e., $\planningfreespace(t) \subseteq \controlfreespace(t) \subseteq \freespace(t)$,  accounting for the robot's body size $\radius$ and the constrained motion gap $\clearance$ between planning and control, see \reffig{fig.over_under_approx}.
For position-controlled, point-sized robots, there is no distinction between the safe planning and control spaces, i.e., $\planningfreespace(t) = \controlfreespace(t) = \freespace(t)$.
\end{remark}

\begin{figure}[t]
\centering
\begin{tabular}{@{\hspace{0.5mm}}c@{\hspace{1mm}}c@{}}
\includegraphics[width = 0.485\columnwidth]{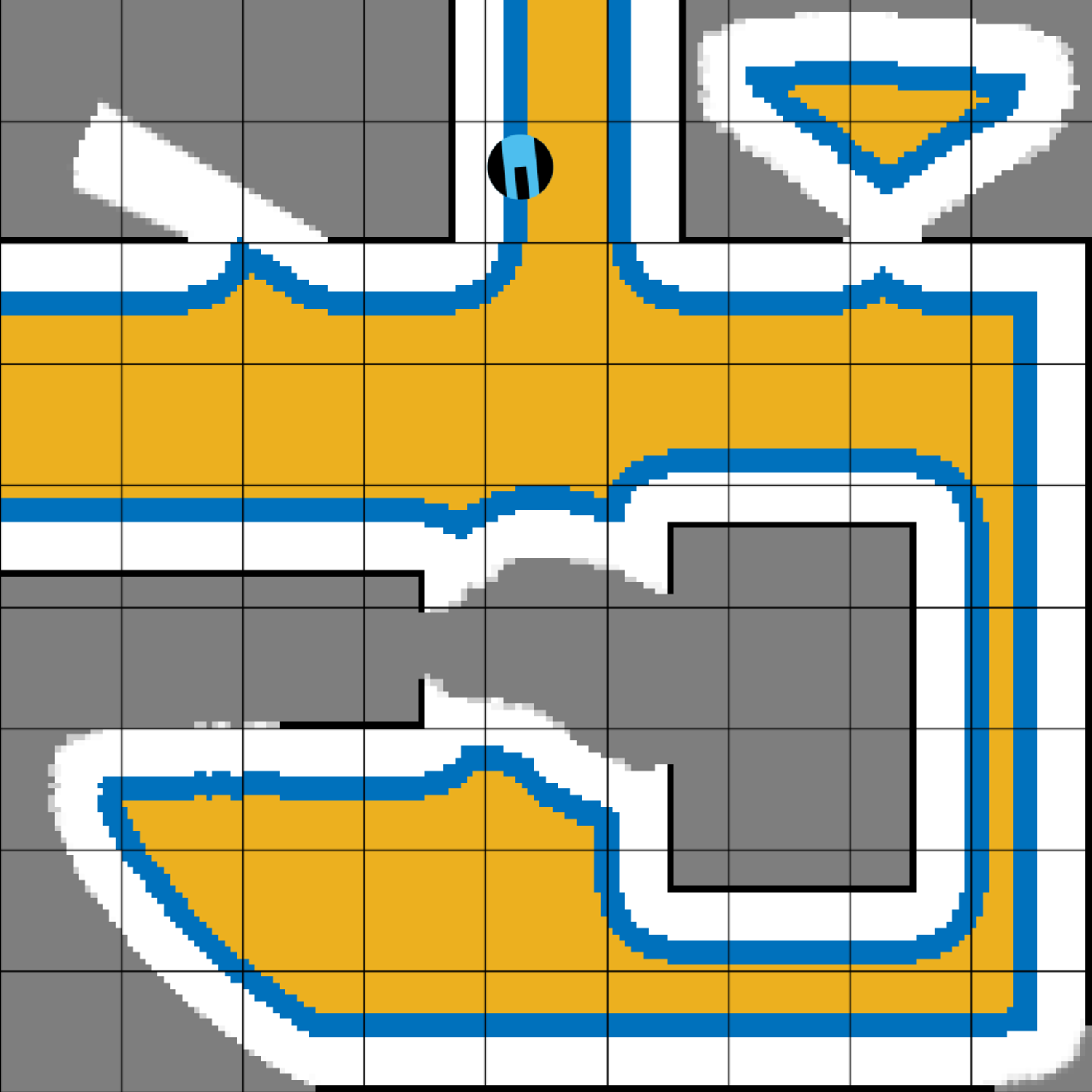} &
\includegraphics[width = 0.485\columnwidth]{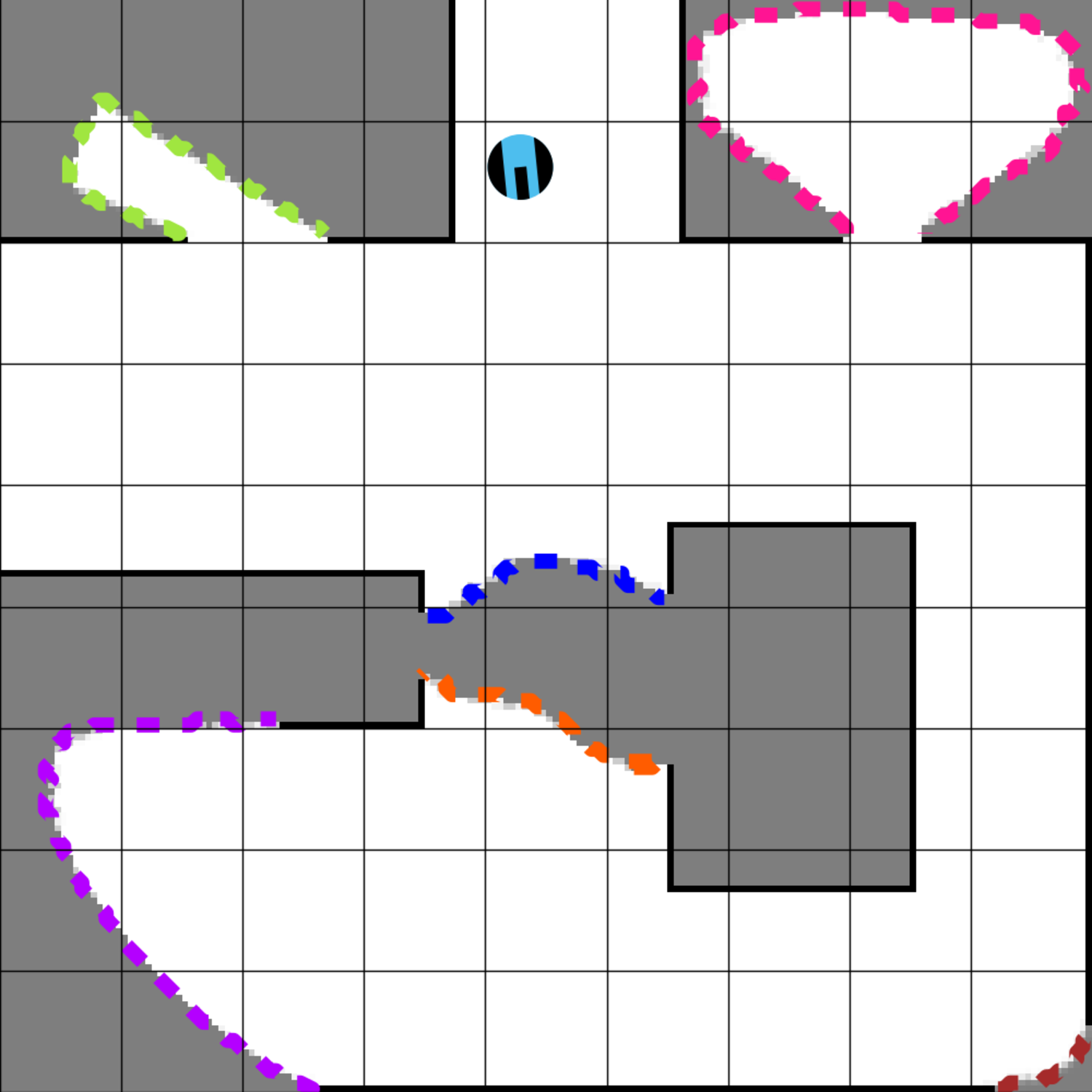} 
\end{tabular}  
\vspace{-2mm}
\caption{(left) The safe planning space, $\planningfreespace(t)$ (yellow), and control space, $\controlfreespace(t)$ (blue), for exploration are constructed by eroding the free space, $\freespace(t)$ (white), based on the robot's body radius $\radius$ and safety clearance $\clearance$.  (right) The connected frontier regions (colored dashed lines) are the boundary between the known free space, $\freespace(t)$ (white), and the unknown space, $\unknownspace(t)$ (gray), while the detected obstacle surfaces correspond to the known occupied space, $\occupiedspace(t)$ (black).  
}
\label{fig.over_under_approx}   
\vspace{-3mm}
\end{figure}

\subsection{Frontier Region Detection}
\label{sec.frontier_region_detection}

A widely used standard active exploration strategy for autonomous mapping is frontier-based exploration \cite{yamauchi_CIRA1997}, which directs a mobile robot to the boundary between the obstacle-free space $\freespace(t)$ and the unknown space $\unknownspace(t)$, also known as the frontier set $\frontierspace(t)$, defined as
\begin{align} \label{eq.frontiers}
    \frontierspace(t) &\ldf  \partial \freespace(t) \cap \partial \unknownspace(t),
\end{align}
to gather new information about the environment and incrementally complete its occupancy map $\map(t)$.
Here, $\partial A$ denotes  the boundary of a set $A$.
For example, the frontier set of an occupancy grid map can be efficiently determined as the set of free grid cells that have at least one neighboring unknown grid cell, based on different connectivity criteria, such as 4-connectivity or 8-connectivity for 2D grids.
A common challenge with a frontier set $\frontierspace(t)$ is that it is often disconnected and consists of a collection of connected regions of nearby frontiers, as illustrated in \reffig{fig.over_under_approx} (right). 
In robotic exploration, rather than treating each frontier point individually or considering all points as a whole, it is common practice to cluster frontier points into coherent, connected sub-regions, e.g., using connected component analysis or clustering, to consistently estimate the amount of new information available around each frontier cluster.
Hence, we assume that availability of a frontier clustering method, denoted by 
\begin{align}
\frontierregionFunc(\map(t)) \mapsto \clist{\frontierregion_1(t), \frontierregion_2(t), \ldots, \frontierregion_m(t)}, \nonumber
\end{align}
that returns a partition $\clist{\frontierregion_1(t), \ldots, \frontierregion_m(t)}$ of the frontier set $\frontierspace(t)$ of an occupancy map $\map(t)$ into  a finite number, $m \!\in\! \N$, of collectively exhaustive (i.e., fully covering all frontier), disjoint (i.e., non-overlapping), connected (i.e., inseparable) subsets, i.e., 
\begin{align}
\bigcup_{i=1}^m \frontierregion_i(t) = \frontierspace(t), \!\! \quad \text{and} \quad \!\!  \frontierregion_i (t) \cap \frontierregion_j(t)  \!=\! \varnothing \,\,\, \forall i \neq j,  \nonumber
\end{align}
and $\overline{P_i} \!\cap\! \overline{\frontierregion_i(t) \!\setminus \!P_i} \!\neq\! \varnothing$ for any nonempty strict \mbox{subset} $P_i \!\subsetneq \! \frontierregion_i(t)$, where $\overline{A}$ denotes the closure of a set $A$.
Accordingly, one can estimate the information content, denoted by $\info_{\map(t)}(\frontierregion_i(t))$ of a frontier region $\frontierregion_i(t)$, for example, using a simple uniform information measure as $\info_{\map(t)}(\frontierregion_i(t)) = 1$ or by using the entropy around the frontier region as a standard information-theoretic novelty measure \cite{stachniss_RoboticMappingExploration2009} as%
\footnote{\label{fn.entropy}The entropy around an $\epsilon$-neighborhood of a frontier region $\frontierregion$ over an occupancy map $\map$ can be determined using a Gaussian kernel as follows:
{\scriptsize
\begin{align*}
\overline{\entropy}(\frontierregion):= \!\int_{\frontierregion} \int_{\ball_\epsilon} & \hspace{-2mm} -\! \map(\pos \!+\! \vect{y}) \log \map(\pos \!+\! \vect{y}) e^{-\frac{\norm{\pos - \vect{y}}^2}{2 \sigma^2 }} \nonumber \\ 
& \hspace{-2.0mm} - \!(1\! -\! \map(\pos \!+\! \vect{y})\!) \log(1\!-\!\map(\pos \!+\! \vect{y})\!) e^{-\frac{\norm{\pos - \vect{y}}^2}{2 \sigma^2 }} \diff \vect{y} \diff \pos 
\end{align*}
}%
where $\ball_\epsilon:= \clist{\vect{y} \in \R^{n}| \, \norm{\vect{y}} \leq \epsilon}$ and $\sigma\!>\!0$ is a Gaussian kernel deviation.
}   
{\small
\begin{align}\label{eq.entropy}
\!\entropy\plist{\frontierregion_i(t)} \!:=\! \!\int\nolimits_{\frontierregion_i(t)}  & \hspace{-4.5mm} -\map(\pos, t) \log (\map(\pos, t)\!) \nonumber
\\ 
& \!\!\! - (1\!-\! \map(\pos, t)) \log (1 \!- \!\map(\pos, t)\!) \diff \pos, \!\!
\end{align}
}%
which is related and proportional with the size (i.e., volume)  of the frontier region $\frontierregion_i(t)$, denoted by  $\volume\plist{\frontierregion_i(t)}$,
{\small
\begin{align}\label{eq.volume}
\volume\plist{\frontierregion_i(t)} := \int_{\frontierregion_i(t)} \!\! 1 \, \diff \pos,
\end{align}    
}%
and corresponds to the number of frontier grids in the frontier region $\frontierregion_i(t)$ on an occupancy grid map \cite{stachniss_RoboticMappingExploration2009}. 
In \reftab{tab.information_measures_for_active_exploration}, we include a list of information measures for active exploration.

\begin{table}[t]
\caption{Information Measures for Active Exploration}
\label{tab.information_measures_for_active_exploration}
\centering
\vspace{-3mm}
\begin{tabular}{@{\hspace{0.5mm}}c@{\hspace{0.5mm}}c@{\hspace{0.5mm}}}
\hline
\hline 
\\[-3mm]
Type & Measure, $\info_{\map(t)}(\frontierregion)$
\\[0.1mm]
\hline
\\[-2.0mm]
\begin{tabular}{@{}c@{}}
Uniform Information
\end{tabular}
& 1
\\[2mm]
\begin{tabular}{@{}c@{}}
Frontier Region Size
\end{tabular}
&  $\volume\plist{\frontierregion} = \int_{\frontierregion} 1 \, \diff \pos$
\\[2mm]
\begin{tabular}{@{}c@{}}
Frontier Region Entropy
\end{tabular}
& $\,\,\entropy\plist{\frontierregion} \!=\!\! \int\limits_{\frontierregion} \!\!\!\scalebox{0.76}{$\begin{array}{l}
 \!\!- \map(\pos,t) \log \map(\pos,t)\\ \,\,\,\,- (1 \!-\! \map(\pos,t)\!)\log(1\!-\!\map(\pos,t)\!)  
\end{array}$} \!\!\diff \pos$
\\
\\[-3mm]
\hline
\hline
\end{tabular}
\vspace{-5mm}
\end{table}

\section{Planning \& Control for Safe Exploration}
\label{sec.planning_control_for_safe_exploration}

In this section, we focus on designing safe actions (i.e., motion planning and control) for active exploration and mapping of unknown environments.
We first describe how we plan a high-level exploration path from any safe starting location to a safe exploration goal, minimizing both the risk of collisions and the distance to unknown regions along the way to support safe exploration.
Next, we describe how to consistently and safely execute a high-level exploration path using safe path-following control for a nonholonomically constrained mobile robot with unicycle dynamics.

\begin{figure}[t]
    \centering
    \begin{tabular}{@{}c@{\hspace{0.01\columnwidth}}c@{\hspace{0.01\columnwidth}}c@{\hspace{0.01\columnwidth}}c@{}}
    \includegraphics[width = 0.31\columnwidth]{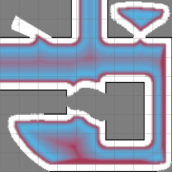} &
    \includegraphics[width = 0.31\columnwidth]{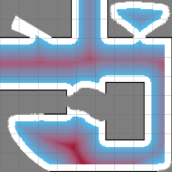} &
    \includegraphics[width = 0.31\columnwidth]{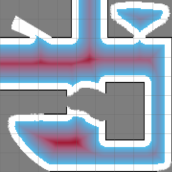} &
    \includegraphics[height =0.31\columnwidth]{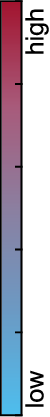}
    \vspace{-2mm}
    \end{tabular}
\caption{The local visit cost (left) for safe and informative path planning in exploration is determined as the ratio of the distance to unknown regions (middle) to the distance to collision (right), balancing obstacle clearance while biasing toward unknown regions.  
In low-cost regions (blue), the robot remains near unexplored areas at a safe distance from collisions.  
}
    \label{fig.safe_information_navigation_costmap_for_exploration}
    \vspace{-2mm}
\end{figure}

\subsection{Reference Path Planning} 
\label{sec.ReferencePathPlanning}

The objective of high-level reference path planning for safe exploration is to determine an optimal, informative, and safe path for a simplified, fully-actuated, position-controlled robot model from any given safe robot position $\pos \in \planningfreespace(t)$ to any selected safe exploration goal (e.g., a frontier) location $\goal \in \planningfreespace(t)$, that avoids collisions along the way while being biased toward unexplored regions. 
Note that $\planningfreespace(t)$ here denotes the safe planning space of a finite-sized circular mobile robot, as defined in \refeq{eq.planningfreespace}, with an additional control margin to ensure continuous and safe control under (e.g., nonholonomic) motion constraints (see \refrem{rem.safe_planning_control_spaces}).
Accordingly, for exploratory path planning over a continuously updated probabilistic occupancy map $\map(t)$ of an unknown static workspace $\workspace$, we define the local visit cost of each location $\pos \in \workspace$ as the ratio of its distance to unknown regions (quantifying how informative it is to visit), denoted by $\distunknown_{\map(t)}(\pos)$, to its distance to collision (quantifying how risky it is to visit), denoted by $\distcoll_{\map(t)}(\pos)$, as 
\begin{align} \label{eq.visitcost}
\visitcost_{\map(t)}(\pos) \!\ldf\! \frac{\distunknown_{\map(t)}(\pos)}{ \distcoll_{\map(t)}(\pos)}. 
\end{align}
For example, as a measure of the distance to the unknown, one can use the inverse of the entropy $\entropy_{\map(t)}(\pos)$ in \refeq{eq.entropy} or the truncated Euclidean distance from $\pos$ to the unknown region $\unknownspace(t)$ of the map $\map(t)$, defined as (see \reffig{fig.safe_information_navigation_costmap_for_exploration}\,(middle))
\begin{align}
\distunknown_{\map(t)}(\pos):= \min \plist{\alpha_{\max}, \min_{\vect{y} \in \unknownspace(t)} \norm{\pos - \vect{y}}},
\end{align}
where $\alpha_{\max} > 0$ is a constant saturation limit that bounds the effect of exploration bias locally.
As a measure of the distance to collision, we use the truncated Euclidean distance from $\pos$ to the exterior (i.e., complement) of the safe planning space $\planningfreespace(t)$ as (see \reffig{fig.safe_information_navigation_costmap_for_exploration}\,(right))
\begin{align} \label{eq.distance_to_collision}
\distcoll_{\map(t)}(\pos) \!:=\!  \min\plist{\beta_{\max}, \min_{\vect{y} \in \workspace \setminus \planningfreespace(t)} \hspace{-1mm} \norm{\pos - \vect{y}}\!},
\end{align}    
where $\beta_{\max}>0$ is a user-defined constant saturation distance that heuristically limits the repulsive effect of the inverse distance field for collision avoidance.
Note that both the distance-to-unknown and distance-to-collision measures can be efficiently computed over an occupancy grid map using standard Euclidean distance transformation methods, as illustrated in \reffig{fig.safe_information_navigation_costmap_for_exploration}.
We include in \reftab{tab.visit_cost} a list of potential measures for distance to the unknown and distance to collision that can be used for optimal safe exploration~planning.

Given such a costmap defined over the occupancy (grid) map $\map(t)$ based on the local visit cost $\visitcost_{\map(t)}(\pos)$ of each position $\pos \in \workspace$, one can use a search-based (e.g., Dijkstra's) or sampling-based (RRT$^*$) planning algorithm to determine an optimal path connecting any safe start position $\pos \in \planningfreespace(t)$ to any goal position $\goal \in \planningfreespace(t)$, minimizing the total cost along the path.
Accordingly, we assume the availability of an optimal safe path planning method, denoted by $\optimalpath_{\map(t)}(\pos, \goal)$, that returns a collision-free path in the safe planning space $\planningfreespace(t)$ such that the travel cost $\travelcost_{\map(t)}(\pos, \goal)$ along the path is minimal, where%
\footnote{
Over an occupancy grid map, the travel cost between a pair of adjacent grids, $\pos_i$ and $\pos_j$, defines a transition (i.e., edge) cost between them based on the average visit cost along the straight-line connection as
\begin{align*}
\travelcost_{\map(t)}(\pos_i, \pos_j) \!=\! \frac{\visitcost_{\map(t)}\!(\pos_i) \!+\! \visitcost_{\map(t)}\!(\pos_j)\!}{2}\norm{\pos_i \!-\! \pos_j}.
\end{align*} 
}
\begin{align}\label{eq.travelcost}
\travelcost_{\map(t)}(\pos, \goal) := \! \int_{\optimalpath_{\map(t)}(\pos, \goal)} \hspace{-16mm} \visitcost_{\map(t)}(\vect{y})\diff \vect{y}. \!\!
\end{align}
This also allows us to determine the reachable subset, denoted by $\planningfreespace(\pos, t)$, of the safe planning space $\planningfreespace(t)$ from any robot position $\pos \in \planningfreespace(t)$ as 
{
\begin{align}\label{eq.reachableplanningfreespace}
\planningfreespace(\pos, t) \!:=\! \clist{\goal \!\in\! \planningfreespace(t) \big | \optimalpath_{\map(t)}(\pos, \goal) \!\neq\! \varnothing}. \!\!\!
\end{align}
}%
Finally, the optimal travel cost $\travelcost_{\map(t)}(\pos, \goal)$ is also used as a geodesic navigation cost for frontier viewpoint selection later, serving as an alternative to uniform or Euclidean distance-based navigation heuristics, as listed in \reftab{tab.navigation_cost_measures}.       

\begin{table}[t]
\caption{Distance to Unknown and Collision \\[-1mm] for Safe Exploration Planning}
\label{tab.visit_cost}
\centering
\vspace{-4mm}
{\scriptsize 
(Measures for determining the visit cost of a position $\pos \in \workspace$ over an occupancy map $\map(t)$ with an obstacle-free  space $\freespace$ and an unknown region $\unknownspace$.)
}
\begin{tabular}{@{}c@{\hspace{3mm}}c@{}}
\\[-2mm]
    \hline
    \hline 
    \\[-3mm]
    Distance to Unknown & Measure, $\distunknown_{\map(t)}(\pos)$
    \\
    \hline
    \\[-2mm]
    \begin{tabular}{@{}c@{}}
        Uniform Distance
    \end{tabular}
    &  $1$
    \\[1mm]
    \begin{tabular}{@{}c@{}}
    Inverse Entropy
    \end{tabular}
    & $\frac{1}{-\map(\pos,t) \log(\map(\pos,t)) - (1- \map(\pos,t)) \log(1 - \map(\pos,t)) }$
    \\[1mm]
    \begin{tabular}{@{}c@{}}
        Euclidean Distance 
    \end{tabular}
    & $\min\plist{\delta, \min_{\vect{y} \in \unknownspace} \norm{\pos - \vect{y}}}$
    \\
    \\[-3mm]
    \hline
    \hline
    \\[-2mm]
    Distance to Collision & Measure, $\distcoll_{\map(t)}(\pos)$
    \\
    \hline
    \\[-2mm]
    \begin{tabular}{@{}c@{}}
    Uniform Distance
    \end{tabular}
    &  $1$
    \\[1mm]
    \begin{tabular}{@{}c@{}}
        Euclidean Distance 
    \end{tabular}
    & $\min\plist{\delta, \min_{\vect{y} \not \in \freespace} \norm{\pos - \vect{y}}}$
    \\
    \hline
    \\[-2.5mm]
    \hline
    \end{tabular}
    \vspace{-\baselineskip}
\end{table}

\subsection{Safe Path Following} 
\label{sec.SafePathFollowing}

Reference path planning with a simplified, fully-actuated, position-controlled robot model is often preferred for ease and computational efficiency in high-level complex decision-making for autonomous robotic exploration of unknown environments. 
However, when it is time to execute the exploration plan on a physical robotic platform, such simplified reference plans often suffer from being unsuitable for precise path-following control due to dynamic feasibility and safety issues for underactuated, nonholonomically constrained (e.g., differential-drive unicycle) mobile robots.
Another common practical challenge for safe path-following control under (nonholonomic) motion constraints is maintaining persistent progress along the path (i.e., the selected guiding local goal on the path for path-following control is expected to always move forward along the path from the start to the end), which may be easily violated because significant deviation from a reference path is sometimes unavoidable or even desirable in order to better utilize the available free space around the path for potential shortcuts and adaptation to more recent observations \cite{isleyen_vandewouw_arslan_IROS2023, isleyen_vandewouw_arslan_cdc2023}.
In this paper, due to their simplicity and wide use in practice, we consider the kinematic unicycle mobile robot model, whose state is represented by its position $\pos \in \R^2$ and forward orientation angle $\ort \in [ -\pi, \pi )$, measured in radians counterclockwise from the horizontal axis, and whose equation of motion is given by
\begin{align} \label{eq.UnicycleDynamics}
\dot{\pos} = \linvel \ovect{\ort} \quad \text{and} \quad \dot{\ort} = \angvel,
\end{align}
where  $\linvel, \angvel \in \R$ are the scalar control inputs that respectively specify the linear and angular velocity of the unicycle robot.
To generate safe robot motion in the control space $\controlfreespace(t)$ in order to realize a reference plan built in the safe planning space $\planningfreespace(t) \subseteq \controlfreespace(t)$ (see \refrem{rem.safe_planning_control_spaces}), we assume the availability of a safe and persistent unicycle path following control approach with the following properties \cite{isleyen_vandewouw_arslan_IROS2023, isleyen_vandewouw_arslan_cdc2023}:

%
%
\begin{table}[t]
    \vspace{1mm}
    \caption{Navigation Cost Measures for Active Exploration}
    \label{tab.navigation_cost_measures}
    \centering
    \vspace{-3mm}
    \begin{tabular}{@{}c@{\hspace{2mm}}c@{}}
    \hline
    \hline 
    \\[-3mm]
    \hspace{0.075\textwidth} Type \hspace{0.05\textwidth} & \hspace{0.05\textwidth} Measure, $\navcost_{\map(t)} (\pos, \vect{y})$ \hspace{0.05\textwidth}
    \\
    \hline
    \\[-2mm]
    \begin{tabular}{@{}c@{}}
        Uniform Cost
    \end{tabular}
    & 1
	\\[1mm]
    \begin{tabular}{@{}c@{}}
        Euclidean Distance
    \end{tabular}
    &  $\norm{\pos - \vect{y}}$
    \\[1mm]
    \begin{tabular}{@{}c@{}}
        Geodesic Distance
    \end{tabular}
    &  $\travelcost_{\map(t)}(\pos, \goal)$
    \\
    \hline
    \\[-2.5mm]
    \hline
    \end{tabular}
    \vspace{-\baselineskip}
\end{table}

\begin{assumption} \label{asm.safe_unicycle_path_following}
\emph{\!(Safe and Persistent Unicycle Path Following)}
For any given safe reference path $\path(\pathparam):[0,1] \rightarrow \planningfreespace(t)$, there exist  a unicycle path-following controller $\ctrl_{\path}(\pos, \ort, \pathparam) \ldf \plist{ \linvel_{\path}(\pos, \ort, \pathparam), \angvel_{\path}(\pos, \ort, \pathparam), \sigma_{\path}(\pos, \ort, \pathparam)}$ such that the closed-loop unicycle path-following dynamics 
\begin{align} \label{eq.SafePathFollowing}
    \!\! \dot{\pos} \!=\! \linvel_{\path}(\pos, \ort, \pathparam) \ovectsmall{\ort}, \,\,\,
    \dot{\ort} \!=\! \angvel_{\path}(\pos, \ort, \pathparam), \,\,\,
    \dot{\pathparam} \!=\!  \sigma_{\path}(\pos, \ort, \pathparam),\!\!
\end{align}
consistently make progress along the reference path (i.e., the reference path parameter $\pathparam \in [0,1]$ is non-decreasing as $\sigma_{\path}(\pos, \ort, \pathparam) \geq 0$) and asymptotically bring any initial safe unicycle state $(\pos, \ort) \in \controlfreespace (t) \times [-\pi, \pi)$ in the safe control space $\controlfreespace(t)$ with $\pos = \path(0)$, as well as the reference path parameter $\pathparam$, from the start $\path(0)$ to the end $\path(1)$ of the reference path, without collisions along the way, i.e.,
\begin{align} \nonumber
&\pos(t) \!\in\! \controlfreespace(t) \!\quad\! \forall t \!\geq\! 0, \!\quad\!
&\lim_{t \rightarrow \infty} \pos(t) = \lim_{t \rightarrow \infty} \path(\pathparam(t)) = \path(\smax).
\end{align}
\end{assumption}
Note that safe and persistent path following, with guaranteed consistent progress along a reference path, is essential for active exploration in unknown environments to avoid deadlocks and livelocks, especially when combined with online replanning, as discussed later in \refsec{sec.exploration_via_online_planning}.  
To ensure a persistent progress along an exploration path, we use a verifiably safe and correct  unicycle path-following control strategy based on feedback motion prediction \cite{isleyen_vandewouw_arslan_IROS2023, isleyen_vandewouw_arslan_cdc2023} and time governors \cite{arslan_arXiv2022}, which determine the unicycle's linear and angular speeds as well as the path parameter rate~as 
\begin{subequations} \label{eq.SafePathFollowing2}
\begin{align}
    &\!\!\!\linvel_{\path}(\pos,\! \theta,\! \pathparam) \!=\! \lingain \max \plist{\! 0 , \ovecTsmalll{\ort}\! \!(\path(\!\!\:\pathparam\!\!\:) \!-\! \pos)\!\!} , \label{eq.SafePathFollowing2a}
    \\
    &\!\!\!\angvel_{\path}(\pos,\! \theta,\! \pathparam) \!=\! \anggain \!\atantwo\plist{\!\nvecTsmalll{\ort}\!\! (\path(\!\!\:\pathparam\!\!\:) \!-\! \pos),\! \ovecTsmalll{\ort}\!\! (\path(\!\!\:\pathparam\!\!\:) \!-\! \pos)\!\! } ,\!\! \label{eq.SafePathFollowing2b}
    \\
    &\!\!\!\sigma_{\path}(\pos,\! \theta,\! \pathparam) \!=\!  \min \plist{ \gain_\safelevel \dist_{\controlfreespace}\! \plist{\motionset_{\!\path(s)}(\pos, \ort)}, \!\gain_\pathparam (\smax \!-\! \pathparam) } ,\!\! \label{eq.SafePathFollowing2c}
\end{align}
\end{subequations}
where $\lingain, \anggain >0$ are fixed positive unicycle control gains for the linear and angular velocity, respectively, and  $\gain_\safelevel, \gain_\pathparam > 0$ are fixed positive control coefficients for online path time parametrization.
Here, $\dist_{\controlfreespace} \plist{\mathcal{X}}$ denotes the distance of a set $\mathcal{X}$ to the exterior of the control space $\controlfreespace(t)$, defined as\footnote{Note that one can also use the minimum distance to the LiDAR-sensed obstacle points  for sensor-based reactive safe path following.}
\begin{align}
\dist_{\controlfreespace}(\mathcal{A}):= \min_{\substack{ \pos \in \mathcal{X} \\ \vect{y} \in \workspace\setminus \controlfreespace(t)}} \norm{\pos - \vect{y}},
\end{align}
and  $\motionset_{\path(s)}(\pos, \ort)$ is a unicycle feedback motion prediction set bounding the closed-loop unicycle position trajectory $\pos(t)$ towards the reference path point $\path(\pathparam)$, defined as \cite{isleyen_vandewouw_arslan_IROS2023}   
{\small
\begin{align} \nonumber
\scalebox{0.9}{$
     \motionset_{ \path(\pathparam)\!}(\pos, \ort) \!\!=\!\!
    \left\{ 
    \begin{array}{@{}l@{}l@{}}
     \conv \plist{ \! \pos, \ball\plist{\path(\!\!\:\pathparam\!\!\:), \!\absval{\!\nvecTsmalll{\ort}\!\!(\path(\!\!\:\pathparam\!\!\:) \!-\! \pos)\!}}\!\!}  & \text{, if } \ovecTsmalll{\ort}\!\!(\path(\!\!\:\pathparam\!\!\:) \!-\! \pos) \!\geq\! 0   \\
      \ball(\path(\!\!\:\pathparam\!\!\:), \norm{\path(\!\!\:\pathparam\!\!\:) \!-\! \pos}) & \text{, otherwise,}
    \end{array} 
    \right.$}
\end{align}
}%
where $\conv$ denotes the convex hull operator.
In brief, the safe unicycle path-following controller in \refeq{eq.SafePathFollowing2} uses the predicted distance to collision, $\dist_{\controlfreespace}\! \plist{\motionset_{\!\path(s)}(\pos, \ort)}$, to increase the path parameter $\pathparam$ in order to persistently make progress along the reference path, while continuously steering the unicycle robot to chase the reference path point $\path(\pathparam)$ as a local control goal \cite{isleyen_vandewouw_arslan_IROS2023}.

\begin{remark}\label{rem.to_be_in_control_not_to_be_in_planning}
\emph{(To be in control or not to be in  planning)} 
Under the safe path-following control (\refasm{asm.safe_unicycle_path_following}), the robot's position $\pos(t)$ always remains within the safe control space $\controlfreespace(t)$, but it does not necessarily stay within the safe planning space $\planningfreespace(t)$, because the safe control space is larger than the safe planning space by a margin equal to the planning-control discrepancy, $\clearance$, i.e., $\controlfreespace(t) = \planningfreespace(t) \oplus \ball(\vect{0}, \clearance) \supset \planningfreespace(t)$ (see \refrem{rem.safe_planning_control_spaces}). 
Therefore, while the robot's position, $\pos(t)$, cannot always be used for (exploration) planning, the current path goal, $\path(\pathparam(t))$, serves as a suitable alternative, as it is safely reachable by $\pos(t)$.
\end{remark}

\begin{figure}[t]
\centering
\begin{tabular}{@{}c@{\hspace{0.01\columnwidth}}c@{}}
    \includegraphics[width = 0.49\columnwidth]{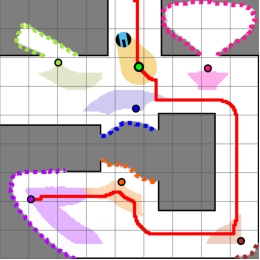} &
    \includegraphics[width = 0.49\columnwidth]{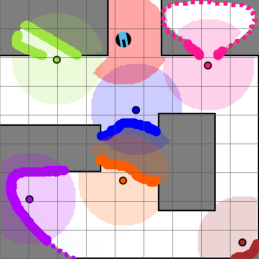}
\end{tabular}
\vspace{-3mm}
\caption{A unicycle mobile robot navigates to the best reachable viewpoint of a frontier region, based on the maximum information utility per navigation cost, by safely and persistently following the reference path (red line) using unicycle feedback motion prediction (yellow cone) towards the local control goal (green dot) on the path.  
(left) The viewpoint sets (colored patches) and the selected viewpoints (colored circles) of the frontier region (colored dashed line).
(right) The immediately available actionable information of a viewpoint is determined by its visible frontiers (highlighted in bold).
}
\label{fig.safe_path_following}
\vspace{-3mm}
\end{figure}

\section{Action-Aware Active Exploration}
\label{sec.action_aware_active_exploration}

In this section, we present an action-aware frontier viewpoint selection strategy for exploration that maximizes information utility per travel cost. 
We then describe its application in two provably correct proactive replanning strategies for adaptive active exploration, as an alternative to the standard  exploration approach with persistent planning, while continuously mapping and safe path following.

\subsection{Action-Aware Frontier Viewpoint Selection} 
\label{sec.ActionAwareFrontierSelection}

Many existing frontier-based exploration strategies for active mapping select the next best viewpoint from the detected frontier points or a statistically representative point (e.g., mean or median) \cite{placed_etal_TRO2023, lluvia_etal_Sensors2021}.
However, selecting a frontier as a viewpoint may be unsafe or unreachable due to unexplored obstacles near the frontier points.
Hence, we start with determining the viewpoint set $\viewpointset_{\map(t)}(\frontierregion_i)$ of a connected frontier region $\frontierregion_i \in \frontierspace(t)=\clist{\frontierregion_1, \ldots, \frontierregion_m}$ as the collection of safe positions in the planning space $\planningfreespace(t)$ from which at least one frontier point in $\frontierregion_i$ is reliably visible over the free space $\freespace(t)$ within the robot's sensing range $\sensorrange>0$ and respecting a certain fixed visibility tolerance $\visibletolerance>0$, as illustrated in \reffig{fig.safe_path_following}\,(left), as
{
\begin{align}\label{eq.viewpointset}
\!\viewpointset_{\map(t)}(\frontierregion_i) \!:=\!  \left \{ \big. \right. &  \!\viewpoint \!\in\! \planningfreespace(t) \big | \,   \exists \frontier \in \frontierregion_i \text{ s. t. }  \nonumber \\ 
&  \hspace{-5mm}  [\frontier, \vect{u}] \! \subseteq \!\freespace(t), \norm{\frontier \!-\! \vect{u}} \!\leq\! \sensorrange \, \forall \vect{u} \!\in\! \ball(\viewpoint, \eta) \!\left. \big. \right \}. \!\!\! 
\end{align}
}%
Note that the visibility tolerance $\visibletolerance$ plays a key role in ensuring robustness and numerical stability through continuous discretization, thereby capturing the effect of discretization resolution in standard occupancy grid maps.
This notion of the viewpoint set also allows for a systematic selection of a safe and reachable viewpoint for a frontier region $\frontierregion_i$ in the reachable planning space $\planningfreespace(\pos, t)$ of a safe robot position $\pos \in \planningfreespace(t)$ that minimizes the total Euclidean distance to the frontier points in $\frontierregion_i$ as%
\footnote{One may also use the volume of the visible frontiers of a viewpoint to select an optimal viewpoint for a frontier region as
\begin{align*}
\view_{\pos, \map(t)}(\frontierregion_i) \!:=\! \argmax_{ \substack{\viewpoint \in \viewpointset_{\map(t)}(\frontierregion_i)\\ \viewpoint \in \planningfreespace(\pos, t)}} \, \volume\plist{\visiblefrontiers_{\map(t)}(\viewpoint) },
\end{align*}
but this is significantly more costly to compute than \refeq{eq.viewpoint}.  
} 
\begin{align} \label{eq.viewpoint}
\view_{\pos, \map(t)}(\frontierregion_i) \ldf \argmin_{ \substack{\viewpoint \in \viewpointset_{\map(t)}(\frontierregion_i)\\ \viewpoint \in \planningfreespace(\pos, t)}} \,\sum_{\frontier \in \frontierregion_i} \norm{\viewpoint - \frontier}.
\end{align}

Moreover, to quantify the immediately available information  content at a viewpoint, as the inverse of the viewpoint set $\viewpointset_{\map(t)}(\frontierspace(t))$, we define the visible frontier set $\visiblefrontiers_{\map(t)}(\vect{v})$ of an occupancy map $\map(t)$ from a safe viewpoint $\viewpoint \in \planningfreespace(t)$, respecting the visibility tolerance $\visibletolerance$, as
{
\begin{align} \label{eq.visiblefrontiers}
\visiblefrontiers_{\map(t)}(\vect{v}) \!\ldf\!\!  \clist{ \frontier \!\in\! \frontierspace(t) \Big | 
\scalebox{0.9}{$
\begin{array}{l}
[\frontier, \!\vect{u}] \! \subseteq \!\freespace(t),\\
\norm{\frontier \!-\! \vect{u}} \!\leq\! \sensorrange
\end{array}$}
  \, \forall \vect{u} \!\in\! \ball(\viewpoint, \eta)\! },
\end{align}
}%
which is strictly bounded as \mbox{$\visiblefrontiers_{\map(t)}(\vect{v})\!\subseteq\!\ball(\viewpoint, \sensorrange \!- \!\eta)$}.
Both the entropy and the volume of the visible frontier set $\visiblefrontiers_{\map(t)}(\vect{v})$, denoted by $\entropy\plist{\visiblefrontiers_{\map(t)}(\vect{v})}$ and $\volume\plist{\visiblefrontiers_{\map(t)}(\vect{v})}$, as defined in \refeq{eq.entropy} and \refeq{eq.volume}, respectively, can be used to heuristically determine the visible information content from the viewpoint $\viewpoint$.  
Since the difference between the entropy and volume of frontier regions is negligible both in theory and practice, we consider the volume of the visible frontiers as an intuitive and interpretable geometric measure of immediately available actionable information at a viewpoint.
Accordingly, we define the immediately available actionable information of a frontier region $\frontierregion_i$ as the volume of the visible frontiers from the viewpoint $\view_{\pos}(\frontierregion_i)$ of the frontier region $\frontierregion_i$ that is safely reachable from the robot position $\pos \in \planningfreespace(t)$, as illustrated in \reffig{fig.safe_path_following}\,(right), as
\begin{align}\label{eq.actionable_info}
\info_{\pos,\map(t)}(\frontierregion_i)\!:=\! \volume(\visiblefrontiers_{\map(t)}(\view_{\pos,\map(t)}(\frontierregion_i)\!)\!).\!\!\!
\end{align} 
Note that the immediately available actionable information $\info_{\pos,\map(t)}(\frontierregion_i)$ is an approximate measure of the exploration rate (under ``teleportation,'' i.e., instantaneous movement) and can be effectively used as an accurate method for eliminating potential frontier regions that are either inaccessible or provide negligible information for exploration, in contrast to more classical frontier region elimination methods based on volume or failed attempts to reach the frontiers.
   
As expected from \refasm{asm.consistent_occupancy_mapping}, if a robot reaches a point near a viewpoint, the continuous consistent occupancy mapping turns an initially informative viewpoint into an uninformative one with no nearby frontiers for all future~times.

\begin{lemma} \label{lem.explored_viewpoints}
\emph{(Explored Viewpoints)}
If there is no frontier around a safe robot position $\pos \in \planningfreespace(t)$ within the sensing range $\sensorrange$, then there are also no reliably visible frontiers $\visiblefrontiers_{\map(t)}(\viewpoint)$ from any neighboring viewpoint $\viewpoint \in \ball(\pos, \eta)$ within the visibility tolerance $\visibletolerance$, i.e.,
\begin{align}
\min_{\frontier \in \frontierspace(t)} \!\! \norm{\frontier \!-\! \pos} > \sensorrange \Rightarrow \visiblefrontiers_{\map(t)}(\viewpoint) \!=\! \varnothing \quad  \forall \viewpoint \!\in\! \ball(\pos, \eta). \!\!
\end{align}   
\end{lemma}
\begin{proof}
The result follows from the definition of the visible frontier set $\visiblefrontiers_{\map(t)}(\viewpoint)$ in \refeq{eq.visiblefrontiers}.
\end{proof}
\noindent An important implication of this lemma is that a mobile robot is required to visit at most a finite set of locations, each apart from the others by the visibility tolerance $\visibletolerance$ (i.e., visiting each grid of an occupancy grid map in the worst case) to complete mapping. 
Hence, we check whether a robot position is near a viewpoint and vice versa using the visibility tolerance $\visibletolerance$ as
\begin{align} \label{eq.isnear}
\isnear(\pos, \viewpoint) = \begin{cases}
1 &  \text{if } \norm{\pos - \viewpoint} \leq \visibletolerance, \\
0 & \text{otherwise}.
\end{cases} 
\end{align}

Therefore, by combining the ideas of the viewpoints of frontier regions and their immediately available actionable information, we consider an action-aware exploration strategy to select a frontier region from the detected frontier regions $\frontierspace(t) = \clist{\frontierregion_1, \ldots, \frontierregion_{m}}$ that has a safely reachable and informative viewpoint from a robot position $\pos \in \planningfreespace(t)$ and maximize the information utility per travel cost as
\begin{align} \label{eq.target_frontier_region}
\!i^*:=  \!\!\! \argmax \limits_{ \substack{ i \in \clist{1,\dots,m} \\ \info_{\pos, \map(t)}\!(\frontierregion_i) > \mininfo  }  } \!\!\! \frac{\info_{\map(t)}(\frontierregion_i)}{ \navcost_{\map(t)} (\pos, \view_{\pos, \map(t)}(\frontierregion_i) )  }, \!\!
\end{align}
where $\info_{\map(t)}(\frontierregion_i)$ measures the utility of the information content of the frontier region $\frontierregion_i$ (see \reftab{tab.information_measures_for_active_exploration}), $\navcost_{\map(t)} (\pos, \view_{\pos,\map(t)}(\frontierregion_i))$ is a measure of the total travel cost to reach the frontier viewpoint $\view_{\pos, \map(t)}(\frontierregion_i)$ from the robot position $\pos$ (see \reftab{tab.navigation_cost_measures}), and $\mu \geq 0$ is a fixed minimum information threshold for the immediately available actionable information $\info_{\pos,\map(t)}(\frontierregion_i)$.
Hence, one can determine if a viewpoint $\viewpoint \in \freespace(t)$ satisfies the minimum actionable information requirement using the volume of the visible frontier set $\visiblefrontiers_{\map(t)}(\viewpoint)$  as
\begin{align} \label{eq.isinformative}
\!\!\isinformative (\map(t),\viewpoint) \!:=\! 
\begin{cases}
1 & \!\!\!\text{if } \!  \volume(\visiblefrontiers_{\map(t)\!}(\viewpoint)\!) \!>\! \mininfo, 
\\
0 & \!\!\!\text{otherwise}.
\end{cases} \!\!\!
\end{align}
Accordingly, we say that the (exploration for) mapping is completed if and only if the immediately available actionable information $\info_{\pos}(\frontierregion_i)$ for each frontier region $\frontierregion_i$ of the frontier set $ \frontierspace(t)=\clist{\frontierregion_1, \ldots, \frontierregion_m}$ of an occupancy map $\map(t)$, relative to a safe robot position $\pos \in \planningfreespace$, is below the minimum information threshold $\mininfo$, i.e.,
\begin{align} \label{eq.iscomplete}
\!\iscomplete(\map(t), \pos) \!:=\! 
\begin{cases}
1 & \!\! \text{if} \!\!\max\limits_{\frontierregion_i \in \frontierspace(t)}\!\!\info_{\pos,\map(t)}(\frontierregion_i) \!\leq\! \mininfo, \\
0 & \!\!\text{otherwise}.
\end{cases} 
\!\!\!
\end{align}
Below, we present three different approaches for active exploration using persistent planning and proactive replanning to achieve provably correct completion of mapping for any unknown environment in finite time, without deadlock (e.g., the robot stops or freezes for a long duration despite the presence of an informative frontier region) or livelock (e.g., the robot continuously switches between informative frontier regions without updating the map).

\subsection{Exploration via Persistent Planning}
\label{sec.exploration_via_persistent_planning}

Once a target frontier viewpoint is selected, a standard active exploration strategy uses persistent planning to determine an optimal reference path to the target viewpoint and then safely and persistently executes this plan until the robot reaches this viewpoint. 
Once the viewpoint is reached, the exploration procedure is repeated with the next best viewpoint until exploration and consequently mapping is completed.
More precisely, for any given safe robot position $\pos \in \planningfreespace(t)$ in the planning space $\planningfreespace(t)$, we apply  the optimal selection of a target frontier region $\frontierregion_{i^*}$ as in \refeq{eq.target_frontier_region} and determine an optimal exploration path in \refalg{alg.exploration_path_planning} from the start position $\pos$ to the frontier viewpoint $\view_{\pos,\map(t)}(\frontierregion_{i^*})$ defined in \refeq{eq.viewpoint} using the reference path planning method described in \refsec{sec.ReferencePathPlanning}, which is denoted by
{\small
\begin{align}
\explorationPlan(\map(t), \pos)\!:=\! \optimalpath_{\map(t)}(\pos, \view_{\pos,\map(t)}(\frontierregion_{i^*}\!)\!). \nonumber
\end{align}
}%
Then, we safely and persistently execute this exploration plan with a unicycle mobile robot using a safe path-following controller described in \refsec{sec.SafePathFollowing} (see \refasm{asm.safe_unicycle_path_following}) until reaching near the selected viewpoint (i.e., the end of the exploration path $\path(1)$ in \refalg{alg.exploration_via_persistent_planning}, line 5).
Once the viewpoint is reached, as discussed in \refrem{rem.to_be_in_control_not_to_be_in_planning}, we repeat this exploration procedure with persistent planning by finding a new exploration path from the current local path goal $\path(\pathparam(t))$ (in \refalg{alg.exploration_via_persistent_planning}, line 8) to the associated next-best frontier viewpoint until the entire exploration and mapping are completed relative to the selected viewpoint at the end of the exploration path $\path(\smax)$ (in \refalg{alg.exploration_via_persistent_planning}, line 6) as determined in \refeq{eq.iscomplete}.
It is important to note that the primary reason for planning from the current local path goal $\path(\pathparam(t))$ is that:
i) the robot position $\pos(t)$ under safe path-following control is guaranteed to remain within the safe control space $\controlfreespace(t)$ but not necessarily within the safe planning space $\planningfreespace(t)$ (see \refrem{rem.to_be_in_control_not_to_be_in_planning}), and
ii) on the other hand, the current path point, $\path(\pathparam(t))$, is always a safe robot position within the safe planning space $\planningfreespace(t)$, and the robot can safely and asymptotically reach $\path(\pathparam(t))$ from $\pos(t)$ within the safe control space $\controlfreespace(t)$ (see \refasm{asm.safe_unicycle_path_following}).

\begin{algorithm}[t] 
    \caption{Reference Path Planning for Exploration \\
    \begin{footnotesize}
    $\mathrm{explorationPlan(\map(t), \pos)}$
    \end{footnotesize} }
    \label{alg.exploration_path_planning}
    \begin{footnotesize}
    \SetKwInOut{Input}{Input}  
    \SetKwInOut{Output}{Output} 
    \Input{%
        $\map(t)$ --  Occupancy Map \\
        $\pos \in \planningfreespace(t)$ -- Start Position \\
    }
    \Output{%
        $\path$ -- Exploration Path
    \\
    \vspace{-1mm}
    }
    \BlankLine
    \hrule
    \BlankLine

    $\plist{ \frontierregion_1,\dots,\frontierregion_m}  \leftarrow \frontierregionFunc(\map(t))$
    \\
    $i^* \leftarrow \argmax \limits_{ \substack{i \in \clist{1,\dots,m} \\ \info_{\pos,\map(t)}(\frontierregion_i) > \mu  }  } \frac{\info_{\map(t)}(\frontierregion_i)}{ \navcost_{\map(t)} (\pos, \view(\frontierregion_i))}$\\
    $\path \leftarrow \optimalpath_{\map(t)}(\pos, \view_{\pos,\map(t)}(\frontierregion_{i^*})) $\\
    \Return{$\path$} \\
    
    \end{footnotesize}
\end{algorithm}

\begin{algorithm}[t!]
    \caption{Exploration via Persistent Planning}
    \label{alg.exploration_via_persistent_planning}
    \begin{footnotesize}
    \SetKwInOut{Input}{Input}  
    \SetKwInOut{Output}{Output} 
    \Input{%
    $\map_0$ -- Initial Occupancy Map\\
    $(\pos_0, \ort_0, \pathparam_0)$ -- Initial Safe State for Path Following \\  
    }
    \Output{%
    $\map^*$ -- Completed Occupancy Map\\
    \\
    \vspace{-1mm}
    }
    \BlankLine
    \hrule
    $\path \leftarrow \pos_0, \pathparam \leftarrow 0  \quad \quad \textit{\% Initial Path}$ \\
    \While{$\mathrm{OccupancyMapping\, \&\,SafePathFollowing}$}{
        $\map(t) \leftarrow \getmap\plist{ \, }$\\
        $(\pos, \ort, \pathparam) \!\leftarrow\! \getState\plist{ \, }$\\

        \If{$\isnear(\pos, \path(\smax))$ }{
            \If{$\iscomplete(\map(t), \path(\smax))$}{
                \Return{$\map^* \leftarrow \map(t)$}\\
            }
            $\path \leftarrow \mathrm{explorationPlan}(\map(t), \path(\pathparam))$\\
            $\pathparam \leftarrow 0$\\
        }
    }
    \end{footnotesize}
\end{algorithm}

\subsection{Exploration via Pro-Active Replanning}

Although exploration with persistent planning in \refalg{alg.exploration_via_persistent_planning} has advantages such as simplicity of implementation and computational efficiency, it suffers from a lack of adaptability to changes and newly discovered areas in the continuously updated occupancy map.
To address this limitation of exploration with persistent planning, we consider two simple proactive replanning strategies, replanning with last-mile prevention and periodic online replanning, for active exploration, aiming to enhance adaptability and mapping efficiency by increasing the exploration rate.
The exploration performance of these methods are systematically compared in numerical simulations and experiments later in \refsec{sec.numerical_simulations_physical_experiments}.

\subsubsection{Exploration with Last-Mile Prevention}

Persistently aiming to reach a selected target viewpoint at the end of an exploration plan, as described in \refalg{alg.exploration_via_persistent_planning}, often slows down the exploration process because the initially available visible frontiers around the viewpoint might already be explored before the viewpoint is reached, as the robot's sensing range is often significantly larger than the robot's body radius.
Hence, to avoid unnecessary last-mile travel at the end of an exploration plan, we propose an adaptive exploration strategy with last-mile prevention in \refalg{alg.exploration_with_last_mile_prevention} (line 5) that proactively triggers the replanning of a new exploration path if the current target viewpoint $\path(1)$ at the end of the exploration path $\path:[0,1] \rightarrow \planningfreespace(t)$ is uninformative with respect to the current occupancy map $\map(t)$, as determined by $\isinformative(\map(t), \path(1))$ in \refeq{eq.isinformative}.
As in persistent planning (\refalg{alg.exploration_via_persistent_planning}, line 8), the replanning of a new exploration plan with last-mile prevention (\refalg{alg.exploration_with_last_mile_prevention}, line 8) is always initiated from the current exploration path point $\path(\pathparam(t))$, since the current robot position $\pos(t)$ is guaranteed to safely reach to the path point $\path(\pathparam(t))$ within the safe control space $\controlfreespace(t)$, and the path point $\path(\pathparam(t))$ lies within the safe planning space $\planningfreespace(t)$.
We observe in \refsec{sec.numerical_simulations_physical_experiments} that this simple last-mile prevention strategy, with continuous evaluation of the informativeness of the current viewpoint, significantly and adaptively reduces unnecessary redundant motion around the end of the exploration plan, contributing to the speed of exploration and mapping.

\begin{algorithm}[t!]
    \caption{Exploration with Last-Mile Prevention}
    \label{alg.exploration_with_last_mile_prevention}
    \begin{footnotesize}
    \SetKwInOut{Input}{Input}  
    \SetKwInOut{Output}{Output} 
    \Input{%
    $\map_0$ -- Initial Occupancy Map\\
    $(\pos_0, \ort_0, \pathparam_0)$ -- Initial Safe State for Path Following \\  
    }
    \Output{%
    $\map^*$ -- Completed Occupancy Map\\
    \\
    \vspace{-1mm}
    }
    \BlankLine
    \hrule
    $\path \leftarrow \pos_0, \pathparam \leftarrow 0 \quad \quad \textit{\% Initial Path}$\\
    \While{$\mathrm{OccupancyMapping\, \&\,SafePathFollowing}$}{
        $\map(t) \leftarrow \getmap\plist{ \, }$\\
        $(\pos, \ort, \pathparam) \!\leftarrow\! \getState\plist{ \, }$\\

        \If{{\normalfont \textbf{not }}$\isinformative(\map(t), \path(\smax))$ }{
            \If{$\iscomplete(\map(t), \path(\smax))$}{
                \Return{$\map^* \leftarrow \map(t)$}\\
            }
            $\path \leftarrow \mathrm{explorationPlan}(\map(t), \path(\pathparam))$\\
            $\pathparam \leftarrow 0$\\
        }
    }
    \end{footnotesize}
\end{algorithm}

\subsubsection{Exploration via Online Planning}
\label{sec.exploration_via_online_planning}

Perhaps the simplest strategy for adaptive exploration planning while continuously mapping is to replan regularly at a certain rate.
While such exploration with online planning offers the advantage of high adaptability to changes in the continuously updated map, it usually suffers from the inherently high computational cost of periodic replanning, as well as potential issues with livelocks due to continuous switching between selected exploration viewpoints without any change or progress in the map.
In \refalg{alg.exploration_via_online_planning}, we present our active exploration strategy via online replanning at a fixed period of $\treplan > 0$ seconds, by simply removing the checks for being near or informative viewpoints, as used in persistent planning (\refalg{alg.exploration_via_persistent_planning}) and last-mile prevention (\refalg{alg.exploration_with_last_mile_prevention}), respectively.
Note that it essentially performs online replanning from the current path goal $\path(\pathparam(t))$ instead of the current robot position $\pos(t)$ (in \refalg{alg.exploration_via_online_planning}, line 7), not only to ensure the existence of an exploration path in the safe planning space $\planningfreespace(t)$, but also to monitor (and guarantee) continuous progress in mapping.
We measure the progress in mapping by the mapping percentage, defined as the ratio of the known volume of the map to its total volume, i.e.,
\begin{align}\label{eq.mapping_percentage}
\!\! \mappct(\map(t)\!) \!:=\! \frac{\volume(\freespace(t)\!\cup\! \occupiedspace(t)\!)}{\volume(\freespace(t)\! \cup\! \occupiedspace(t)\!) \!+\! \volume(\unknownspace(t)\!)}, \!\!\!
\end{align}
where $\freespace(t), \occupiedspace(t), \unknownspace(t)$ denote the free, occupied, and unknown parts of an occupancy map $\map(t)$ as defined in \refeq{eq.free_occupied_unknown_spaces}, respectively.

\begin{algorithm}[t]
    \caption{Exploration via Online Planning}
    \label{alg.exploration_via_online_planning}
    \begin{footnotesize}
    \SetKwInOut{Input}{Input}  
    \SetKwInOut{Output}{Output} 
    \BlankLine
    \Input{%
    $\map_0$ -- Initial Occupancy Map\\
    $(\pos_0, \ort_0, \pathparam_0)$ -- Initial Safe State for Path Following \\
    $\treplan$ -- Online Planning Period  
    }
    \Output{%
    $\map^*$ -- Completed Occupancy Map\\
    \\
    }
    \BlankLine
    \hrule
    $\path \leftarrow \pos_0, s \leftarrow 0 \quad \quad \textit{\% Initial Path}$\\
    \While{$\mathrm{OccupancyMapping\, \&\,SafePathFollowing}$}{
        $\map(t) \leftarrow \getmap\plist{ \, }$\\
        $(\pos, \ort, \pathparam) \!\leftarrow\! \getState\plist{ \, }$\\

        \If{$\iscomplete(\map(t), \path(\smax))$}{
            \Return{$\map^* \leftarrow \map(t)$}\\
        }
        $\path \leftarrow \explorationPlan(\map(t), \path(\pathparam))$ \\
        $\pathparam \leftarrow 0$ \\
        $\mathrm{pause}( \treplan )$  \\

    }
    \end{footnotesize}
\end{algorithm}

\vspace{-0.5\baselineskip}
\begin{remark}
\emph{(Exploration Replanning \& Completion with a Previous Plan)}
Since the robot position $\pos(t)$ is not always within the safe planning space $\planningfreespace(t)$ (\refrem{rem.to_be_in_control_not_to_be_in_planning}), exploration replanning should begin from a safely navigable path point (e.g., the current path point $\path(\pathparam(t))$). 
In contrast, the completion of exploration can be determined using any point on the current exploration path (e.g., $\path(\smin)$, $\path(\pathparam(t))$, or $\path(\smax)$).   
\end{remark}

\subsubsection{Guaranteed Finite-Time Exploration \& Mapping}

As expected, under \refasm{asm.safe_unicycle_path_following} and due to \reflem{lem.explored_viewpoints}, persistent planning ensures the completion of exploration in finite time due to its exhaustive nature.
In contrast, online planning is known to suffer from potential livelocks with improper navigation cost selection (e.g., Euclidean distance), which can be mitigated by using the geodesic navigation cost.

\begin{proposition}\label{prop.finite_time_mapping}
\emph{(Exploration and Mapping in Finite Time)}    
Starting at time $t\!=\!0$ with any initial occupancy map $\map(0)\!=\! \map_0$ of an unknown bounded static workspace $\workspace$ and any initial safe path-following state $(\pos_0, \ort_0, \pathparam_0) \in \planningfreespace(0) \times [-\pi, \pi) \times [\smin, \smax]$, exploration with persistent planning (\refalg{alg.exploration_via_persistent_planning}) and last-mile prevention (\refalg{alg.exploration_with_last_mile_prevention}), under consistent occupancy mapping (\refasm{asm.consistent_occupancy_mapping}) and safe path-following control (\refasm{asm.safe_unicycle_path_following}), ensures the safe completion of mapping in the sense of \refeq{eq.iscomplete} within finite time; 
however, exploration with online planning (\refalg{alg.exploration_via_online_planning}) guarantees finite-time completion of mapping if the navigation cost of the action-aware frontier region selection in \refeq{eq.target_frontier_region} is set to the actual travel (i.e., geodesic) cost of the reference path planner in \refeq{eq.travelcost}, i.e., $\navcost = \travelcost$. 
\end{proposition}
\begin{proof}
We have from \refasm{asm.consistent_occupancy_mapping} and \reflem{lem.explored_viewpoints} that the robot must visit a finite number of safe positions, each separated by at least the visibility tolerance $\visibletolerance$, in order to explore all potential safe viewpoints $\viewpoint \in \workspace$ with nonzero immediately available actionable information $\volume(\visiblefrontiers_{\map(t)}(\viewpoint))$.
In fact, a complete mapping in the sense of \refeq{eq.iscomplete} only requires visiting such potential viewpoints with a minimum information content $\mininfo$, i.e., $\volume(\visiblefrontiers_{\map(t)}(\viewpoint)) > \mininfo$.
Also, note that, due to \refasm{asm.consistent_occupancy_mapping} and \reflem{lem.explored_viewpoints}, once a viewpoint becomes uninformative as in \refeq{eq.isinformative} (i.e., $\volume(\visiblefrontiers_{\map(t)}(\viewpoint)) \leq \mininfo$), it remains uninformative for all future times (i.e., $\volume(\visiblefrontiers_{\map(t')}(\viewpoint)) \leq \mininfo$ for all $t'\geq t$).
Hence, regardless of the target frontier region selection in \refeq{eq.target_frontier_region}, exploration with persistent planning, with or without last-mile prevention (Algorithms \ref{alg.exploration_via_persistent_planning} \& \ref{alg.exploration_with_last_mile_prevention}), directly ensures the finite-time completion of mapping in the sense of \refeq{eq.iscomplete}, because there are a finite number of potentially  informative viewpoints to visit, and  under the path-following control (\refasm{asm.safe_unicycle_path_following}), the robot safely reaches the $\visibletolerance$-neighborhood of each informative viewpoint in finite time and turns them into uninformative.
Therefore, one can conclude that the mapping percentage $\mappct(\map(t)\!)$ in \refeq{eq.mapping_percentage} is non-decreasing and guaranteed to strictly increase in a finite time before mapping is completed with persistent and preventive planning.       

Similarly, the finite-time completion of exploration with online planning based on geodesic travel cost as a navigation heuristic (i.e., $\navcost = \travelcost$ in \refeq{eq.target_frontier_region}) can also be proven using the mapping percentage $\mappct(\map(t))$ as an objective function to demonstrate continuous and persistent improvement over time.
Note that, under the consistent occupancy mapping assumption (\refasm{asm.consistent_occupancy_mapping}), the mapping percentage is always non-decreasing over time, and a strictly increasing mapping percentage implies that some informative viewpoints become uninformative over time since some frontiers are explored. 
Hence, the finite-time exploration with online planning can be established by showing that the mapping percentage cannot remain constant indefinitely if mapping is incomplete.
We prove this by contradiction.
 
Suppose that $\mappct(\map(t))$ remains constant for all future times $t'\geq t$, which means that $\map(t) = \map(t')$ due to \refasm{asm.consistent_occupancy_mapping}. 
Then, the frontier regions $\frontierspace(t)=\plist{\frontierregion_1, \ldots, \frontierregion_m}$, their information contents, measured by $\info_{\map(t)}(\frontierregion_i)$, and their viewpoint $\viewpoint_i = \view_{\path(\pathparam(t)), \map(t)}(\frontierregion_i)= \view_{\pos_0, \map(t)}(\frontierregion_i)$  also remain constant. 
Consequently, according to online exploration planning in \refalg{alg.exploration_via_online_planning} and the action-aware frontier region selection in \refeq{eq.target_frontier_region}, the robot periodically selects an optimal frontier region, plans an exploration path, and moves toward the viewpoint of the selected frontier region.
Since the information utility remains the same and the path-following control is persistent, the information utility per navigation cost is nonincreasing for the selected viewpoint, and a potential transition to another viewpoint results in a strict decrease in utility per cost.
One can also observe that such a transition may only occur for a finite time and only with alternative viewpoints at a lower navigation cost, as follows. 
At time $t$, by definition \refeq{eq.target_frontier_region}, the optimal selected frontier region $\frontierregion_{i^*}$ satisfies  for any $j \neq i^*$ that 
{\small
\begin{align*}
\frac{\info_{\map(t)}(\frontierregion_j)}{\info_{\map(t)}(\frontierregion_{i^*})} \leq \frac{\travelcost_{\map(t)}(\path(\pathparam(t)), \viewpoint_j)}{\travelcost_{\map(t)}(\path(\pathparam(t), \viewpoint_{i^*})},
\end{align*}
}%
where the left hand side (i.e., the ratio of the information utility of frontier regions) stays constant over time.
If the exploration path is optimal towards $\viewpoint_{i^*}$, i.e., $\path = \optimalpath_{\map(t)}(\path(0), \viewpoint_i^*)$ and keep safely and persistently followed (see \refasm{asm.safe_unicycle_path_following}), then we have the geodesic distance decay towards $\viewpoint_{i^*}$ for any future time $t'\geq t$ with the following property
{\scriptsize
\begin{align*}
&\frac{\travelcost_{\map(t)}(\path(\pathparam(t)), \viewpoint_j)}{\travelcost_{\map(t)}(\path(\pathparam(t)), \viewpoint_{i^*})} \geq 1
\\
& \quad  \Longrightarrow \frac{\travelcost_{\map(t')}(\path(\pathparam(t')), \viewpoint_j)}{\travelcost_{\map(t')}(\path(\pathparam(t')), \viewpoint_{i^*})} \geq \frac{\travelcost_{\map(t)}(\path(\pathparam(t)), \viewpoint_j)}{\travelcost_{\map(t)}(\path(\pathparam(t)), \viewpoint_{i^*})}.
\end{align*}  
}%
Hence, the robot may change its optimal frontier region selection along the way if it finds a better frontier region with a strictly lower information utility per travel cost and at a strictly lower travel cost. 
Since the number of such transitions is at most equal to the number of frontier regions, the robot, in finite time, either exhausts all potential frontier regions or converges to the $\visibletolerance$-neighborhood of its frontier region viewpoint.
This contradicts the assumption that $\mappct(\map(t))$ remains constant, since reaching the $\visibletolerance$-neighborhood of an informative viewpoint in finite time ensures a strict increase in the mapping percentage $\mappct(\map(t))$, which completes the proof. 
\end{proof}

\section{Numerical Simulations \& Experiments}
\label{sec.numerical_simulations_physical_experiments}

In this section, we evaluate the effectiveness of our action-aware active exploration framework through numerical simulations%
\footnote{In ROS-Gazebo numerical simulations, we use a circular unicycle mobile robot with a body radius of 0.35m, equipped with a 2D 360$^\circ$ laser range finder sensor that generates 360 samples with a maximum range of 1.5m at 10Hz.
The robot's pose is updated at 10Hz using a simulated motion capture system, and the robot is controlled using the unicycle-path-following controller in \refeq{eq.SafePathFollowing2} at 10Hz with path following coefficients of  $\gain_\safelevel = 1.0$ and $\gain_\pathparam = 1$, robot control gains of $\lingain = 1$ and $\anggain = 2$, and a maximum linear and angular velocity of 1.0\,m/s and 1.0\,rad/s, respectively.} 
and physical experiments. 
We systematically investigate the roles of the information measure (\reftab{tab.information_measures_for_active_exploration}), navigation cost measure (\reftab{tab.navigation_cost_measures}), and exploration planning strategy (Algorithm \ref{alg.exploration_via_persistent_planning}-\ref{alg.exploration_via_online_planning}) in mapping efficiency, quantified as the mapping percentage relative to the distance traveled.


\begin{figure}[t]
    \centering
    \begin{tabular}{@{}c@{\hspace{0.01\columnwidth}}c@{}}
    \raisebox{.005\height}{\includegraphics[width = 0.295\columnwidth]{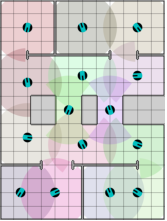}} &
    \includegraphics[width = 0.695\columnwidth]{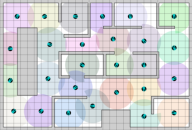}
    \\[-1mm]
    \end{tabular}
    \vspace{-2mm}
    \caption{Initial robot poses for safe exploration and mapping in numerical simulations using (left) 12m$\times$16m and (right) 28m$\times$19m office-like environments, where the robot's sensing regions are depicted as colored patches.
    }
    \label{fig.initial_conditions}
    \vspace{-1.5\baselineskip}
    \end{figure}

\begin{figure}[t]
    \centering
    \begin{tabular}{@{}c@{\hspace{0.005\columnwidth}}c@{\hspace{0.004\columnwidth}}c@{}}
        \rotatebox{90}{\parbox{0.35\columnwidth}{\centering \scriptsize{Information Only} }} &
        \includegraphics[width = 0.36\columnwidth, angle = 90]{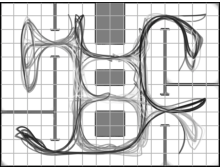} &
        \includegraphics[width = 0.36\columnwidth, angle = 90] {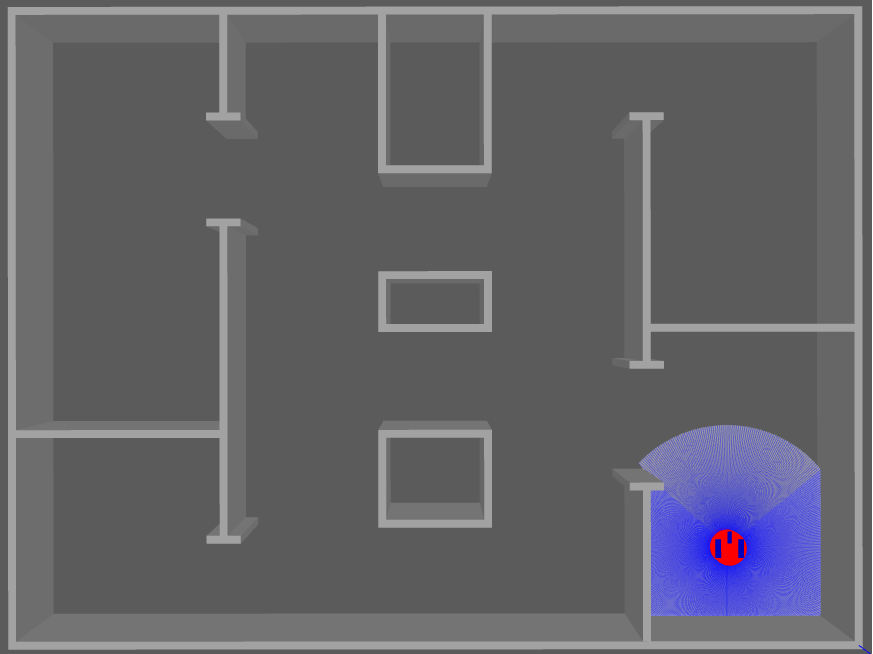}
    \\[-1mm]
        \rotatebox{90}{\parbox{0.35\columnwidth}{\centering \scriptsize{Uniform Info} }} &
        \includegraphics[width = 0.36\columnwidth, angle = 90]{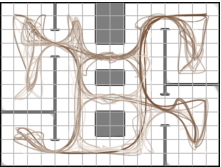} &
        \includegraphics[width = 0.36\columnwidth, angle = 90]{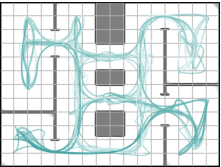}
    \\[-1mm]
        \rotatebox{90}{\parbox{0.35\columnwidth}{\centering \scriptsize{Frontier Info} }} &
        \includegraphics[width = 0.36\columnwidth, angle = 90]{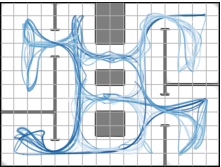} &
        \includegraphics[width = 0.36\columnwidth, angle = 90]{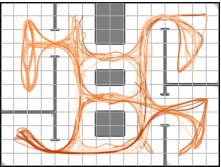} 
    \\[-2mm]
        \rotatebox{0}{\centering \scriptsize{\phantom{.}} }&
        \rotatebox{0}{\parbox{0.36\columnwidth}{\centering \scriptsize{Euclidean Distance} }} &
        \rotatebox{0}{\parbox{0.36\columnwidth}{\centering \scriptsize{Geodesic Distance} }}
    \\
        \rotatebox{0}{\centering \scriptsize{\phantom{.}} } &
        \includegraphics[width = 0.4\columnwidth]{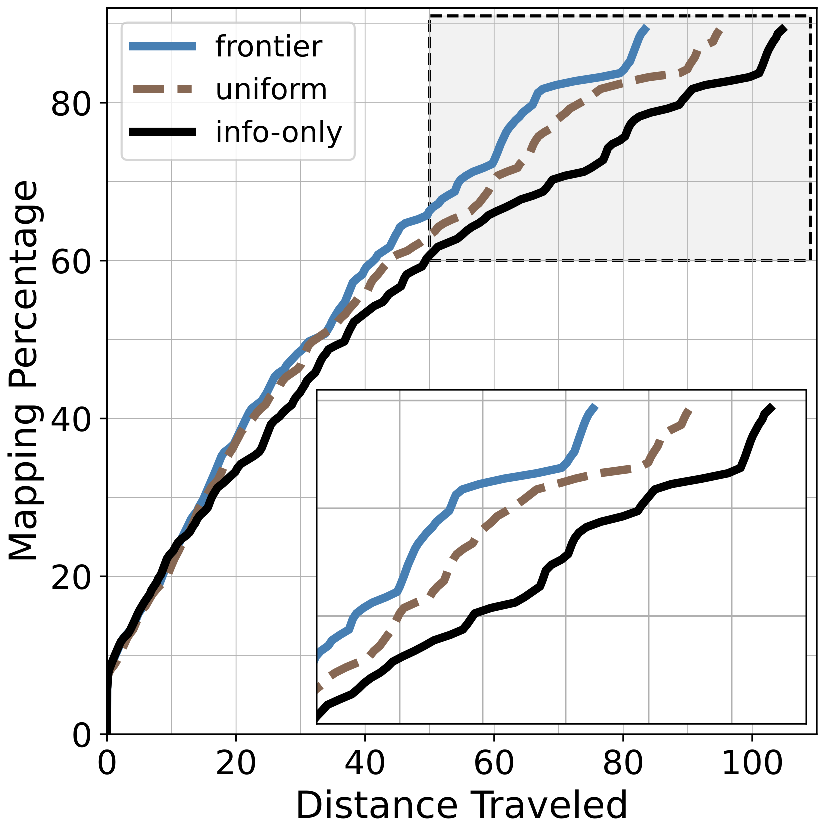} &
        \includegraphics[width = 0.4\columnwidth]{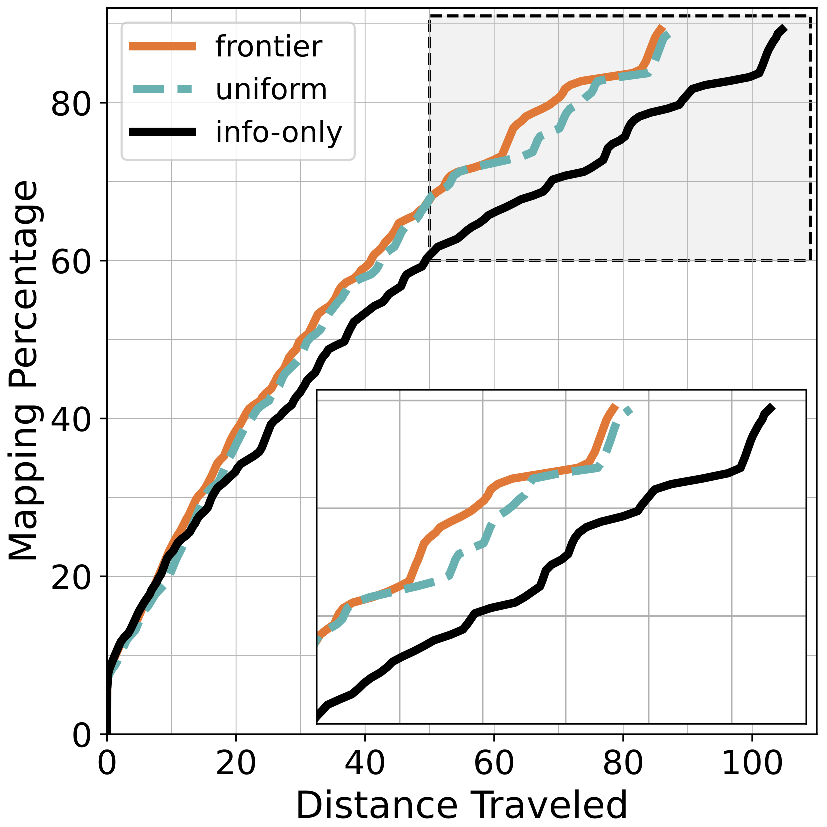}
    \\[-3mm]
    \end{tabular}
    \vspace{-1mm}
    \caption{Unicycle robot motion while exploring and mapping a 12m$\times$16m simulated Gazebo environment with a laser scanner, using persistent planning with different information utilities (top: information-only, middle: uniform information, bottom: frontier size) and navigation costs (left: Euclidean distance, right: geodesic distance). 
Mapping percentage versus distance traveled (bottom) shows that action-aware exploration with geodesic distance outperforms alternatives, and information utility supports the less accurate Euclidean distance but is less influential for geodesic distance. 
}
    \label{fig.effect_of_information_measure}
    \vspace{-3mm}
    \end{figure}

\subsection{Effect of Information Measure on Active Exploration}

We first investigate how uniform and nonuniform-frontier-based information measures in \reftab{tab.information_measures_for_active_exploration} influence exploration efficiency under persistent planning (\refalg{alg.exploration_via_persistent_planning}) using different navigation cost measures (uniform, Euclidean, and geodesic in \reftab{tab.navigation_cost_measures}). 
We consider a 12m$\times$16m office-like environment in Gazebo numerical simulations with a set of relatively uniform and minimally overlapping initial poses, as illustrated in \reffig{fig.initial_conditions} (left) and (see \reffig{fig.effect_of_information_measure}, top-right).
\reffig{fig.effect_of_information_measure} presents the resulting robot position trajectories for all exploration trials over the constructed average occupancy grid map, where darker path segments indicate longer distances traveled.
As expected from \refprop{prop.finite_time_mapping}, exploration with persistent planning completes mapping in finite time regardless of initial robot poses, information measures, or navigation cost selections;
however, exploration efficiency may vary significantly in terms of the mapping percentage over the distance traveled.
As observed in \reffig{fig.effect_of_information_measure} (bottom), exploration using only the information measure (with a constant uniform navigation cost, ignoring the role and cost of actions in exploration) is the slowest and performs the worst in terms of exploration efficiency,
as also indicated by the longer darker trajectories in \reffig{fig.effect_of_information_measure} (top-left).
On the other hand, action-aware exploration with uniform or nonuniform information measures always reduces the average travel distance needed to complete mapping, \reffig{fig.effect_of_information_measure}\,(bottom).
In particular, exploration with the Euclidean navigation cost noticeably benefits from the frontier-size-based information measure compared to the uniform measure, as it reduces the travel distance required to complete mapping (see \reffig{fig.effect_of_information_measure}, bottom-left), suggesting that an informative information measure helps compensate for inaccuracies in travel cost estimation.
Exploration with a more accurate travel cost estimation using the geodesic distance shows negligible  difference for uniform and frontier-size-based information measures, as seen in \reffig{fig.effect_of_information_measure} (bottom-right).
Overall, we observe that an informative information measure positively contributes to the performance of action-aware exploration;  
however, the accuracy of the estimated navigation cost has a significant influence in exploration efficiency. 
This can be explained by the fact that the robot inevitably needs to visit certain parts of the environment in order to complete mapping, regardless of the amount of information available.

\begin{figure}[t]
    \centering
    \begin{tabular}{@{}c@{\hspace{0.005\columnwidth}}c@{\hspace{0.005\columnwidth}}c@{\hspace{0.005\columnwidth}}c@{}}
        \rotatebox{90}{\phantom{label}} &
        \rotatebox{0}{\parbox{0.3\columnwidth}{\centering \scriptsize{Persistent Planning} }} &
        \rotatebox{0}{\parbox{0.3\columnwidth}{\centering \scriptsize{Preventive Planning} }} &
        \rotatebox{0}{\parbox{0.3\columnwidth}{\centering \scriptsize{Online Planning} }}
    \\[-1mm]
        \rotatebox{90}{\parbox{0.46\columnwidth}{\centering \scriptsize{Euclidean Navigation Cost} }} &
        \includegraphics[width = 0.46\columnwidth, angle = 90]{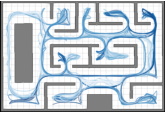} &
        \includegraphics[width = 0.46\columnwidth, angle = 90]{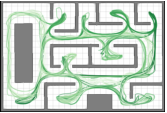} &
        \includegraphics[width = 0.46\columnwidth, angle = 90]{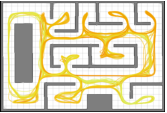}
    \\[-1mm]
        \rotatebox{90}{\parbox{0.46\columnwidth}{\centering \scriptsize{Geodesic Navigation Cost} }} &
        \includegraphics[width = 0.46\columnwidth, angle = 90]{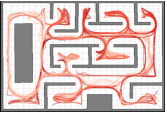} &
        \includegraphics[width = 0.46\columnwidth, angle = 90]{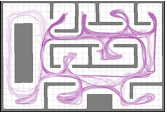} &
        \includegraphics[width = 0.46\columnwidth, angle = 90]{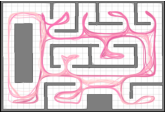}
    \\[-1mm]
        \rotatebox{90}{\parbox{0.305\columnwidth}{\phantom{Rotated Label}}} &
        \includegraphics[width = 0.305\columnwidth]{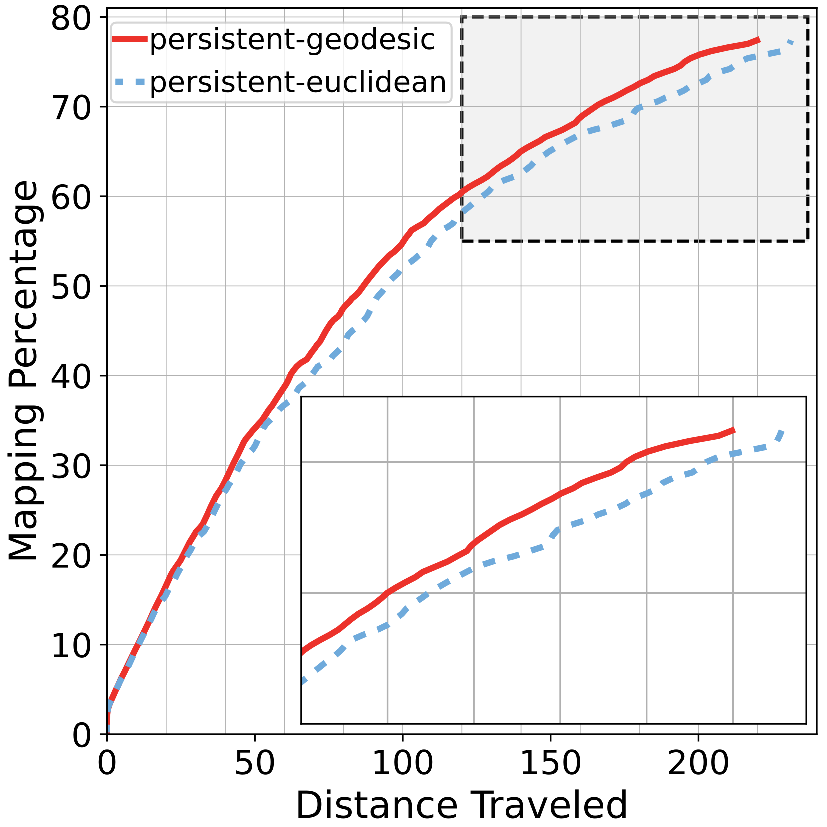}  &
        \includegraphics[width = 0.305\columnwidth]{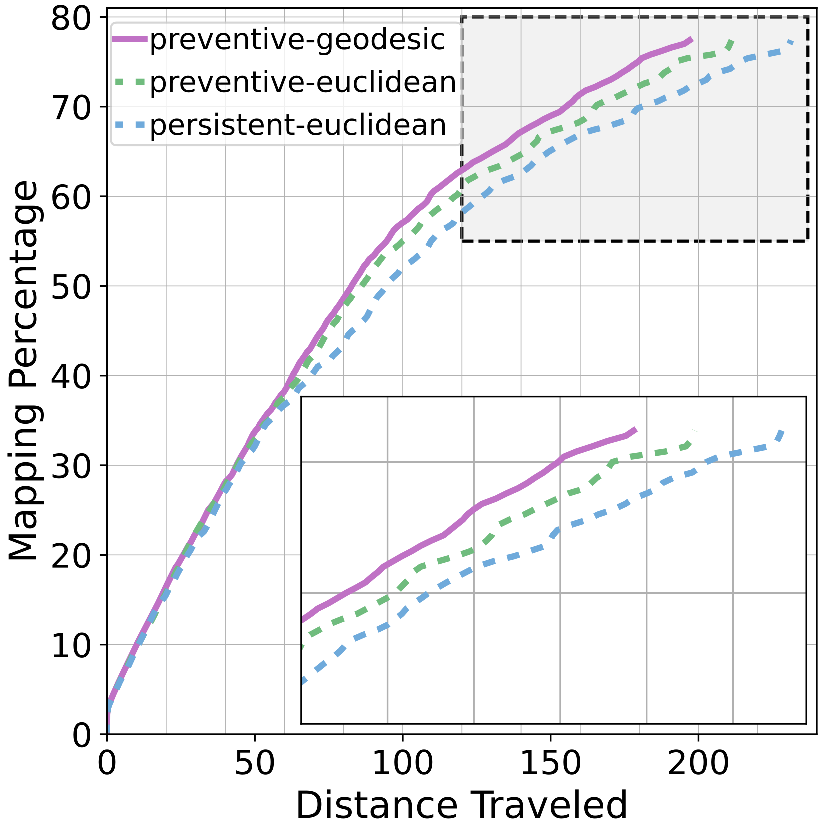}  &
        \includegraphics[width = 0.305\columnwidth]{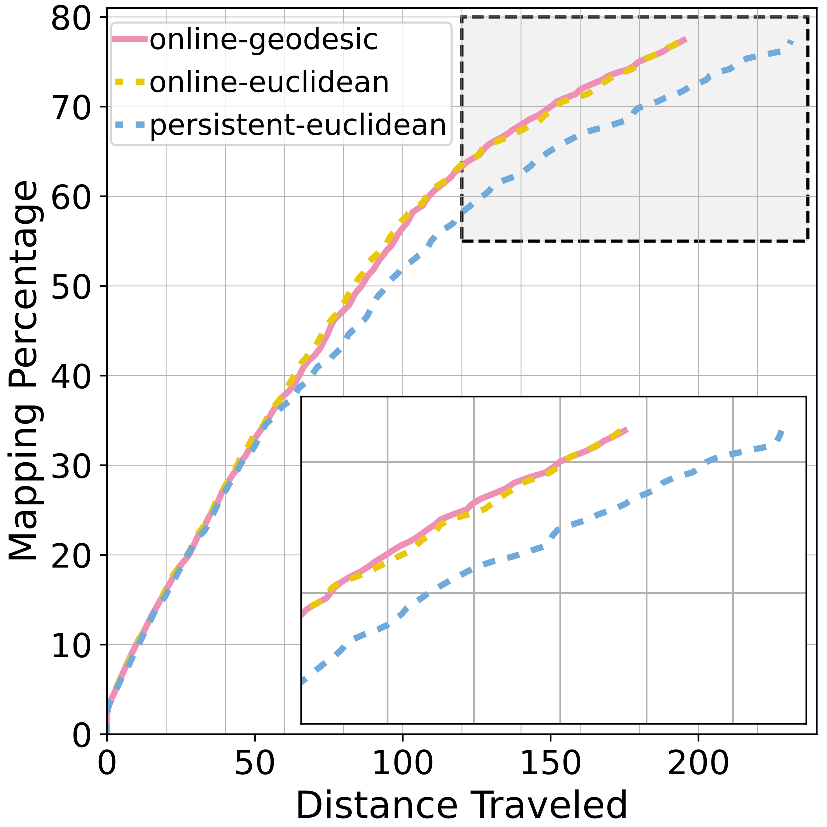}
    \\[-2mm]
    \end{tabular}
    \caption{Unicycle robot motion while exploring and mapping a 28m$\times$19m office-like simulated environment based on the frontier-size base information utility, using different planning (left: persistent, middle: preventive, right: online) and navigation costs (upper: Euclidean distance, lower: geodesic distance). 
Mapping percentage versus distance traveled (bottom) shows that planning with geodesic distance almost always outperforms Euclidean distance, and last-mile preventive planning using geodesic distance offers the best adaptability, computational cost, and exploration efficiency.     
%
    }
    \vspace{-2mm}
    \label{fig.effect_of_navigation_cost_and_replanning_strategy}
    \end{figure}

\subsection{Effect of Navigation Cost and Pro-Active Replanning}

\begin{figure*}[!t]
    \centering
    \begin{tabular}{@{}c@{\hspace{0.01\columnwidth}}c@{\hspace{0.01\columnwidth}}c@{\hspace{0.01\columnwidth}}c@{\hspace{0.01\columnwidth}}c@{}}
    \includegraphics[width = 0.195\textwidth]{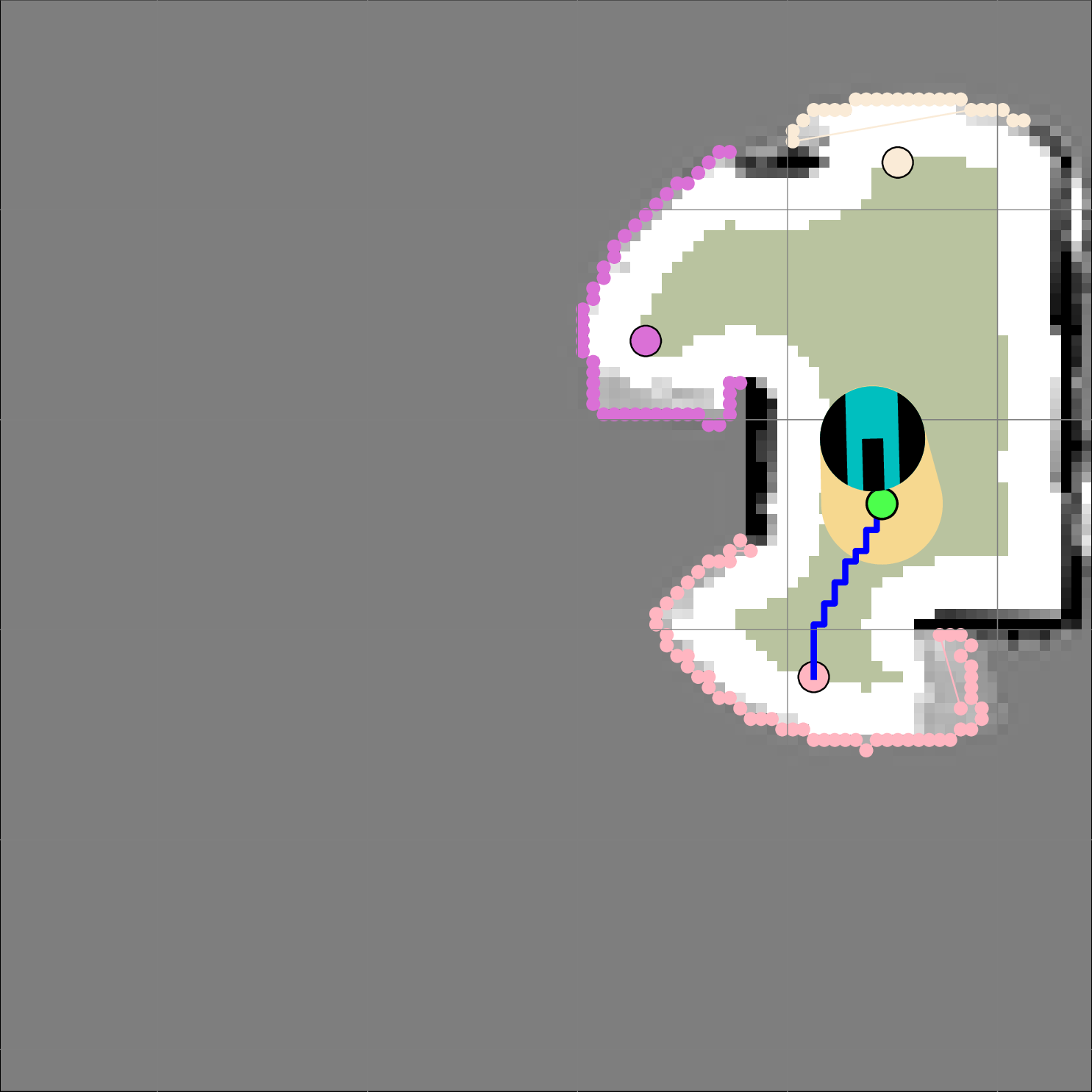} &
    \includegraphics[width = 0.195\textwidth]{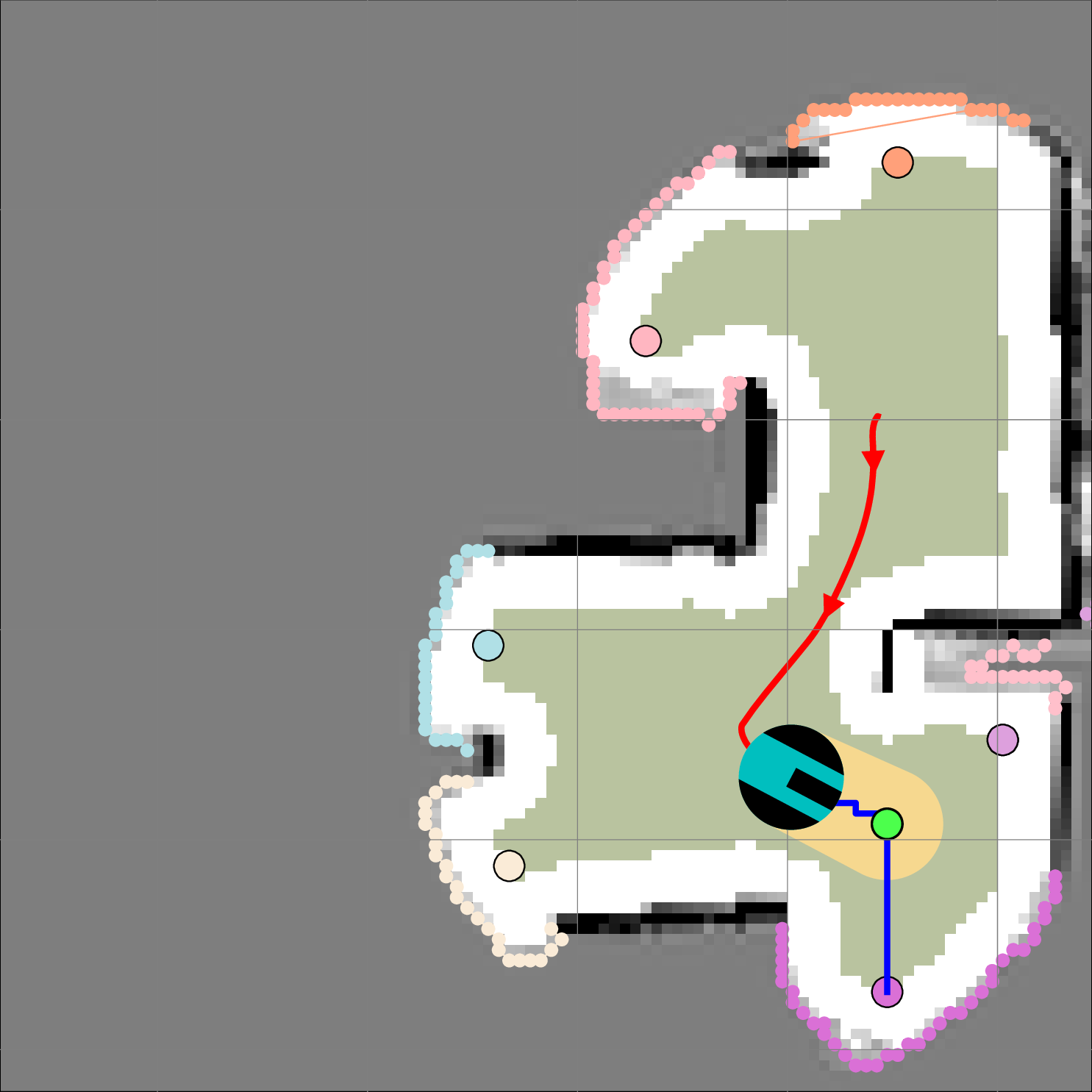} &
    \includegraphics[width = 0.195\textwidth]{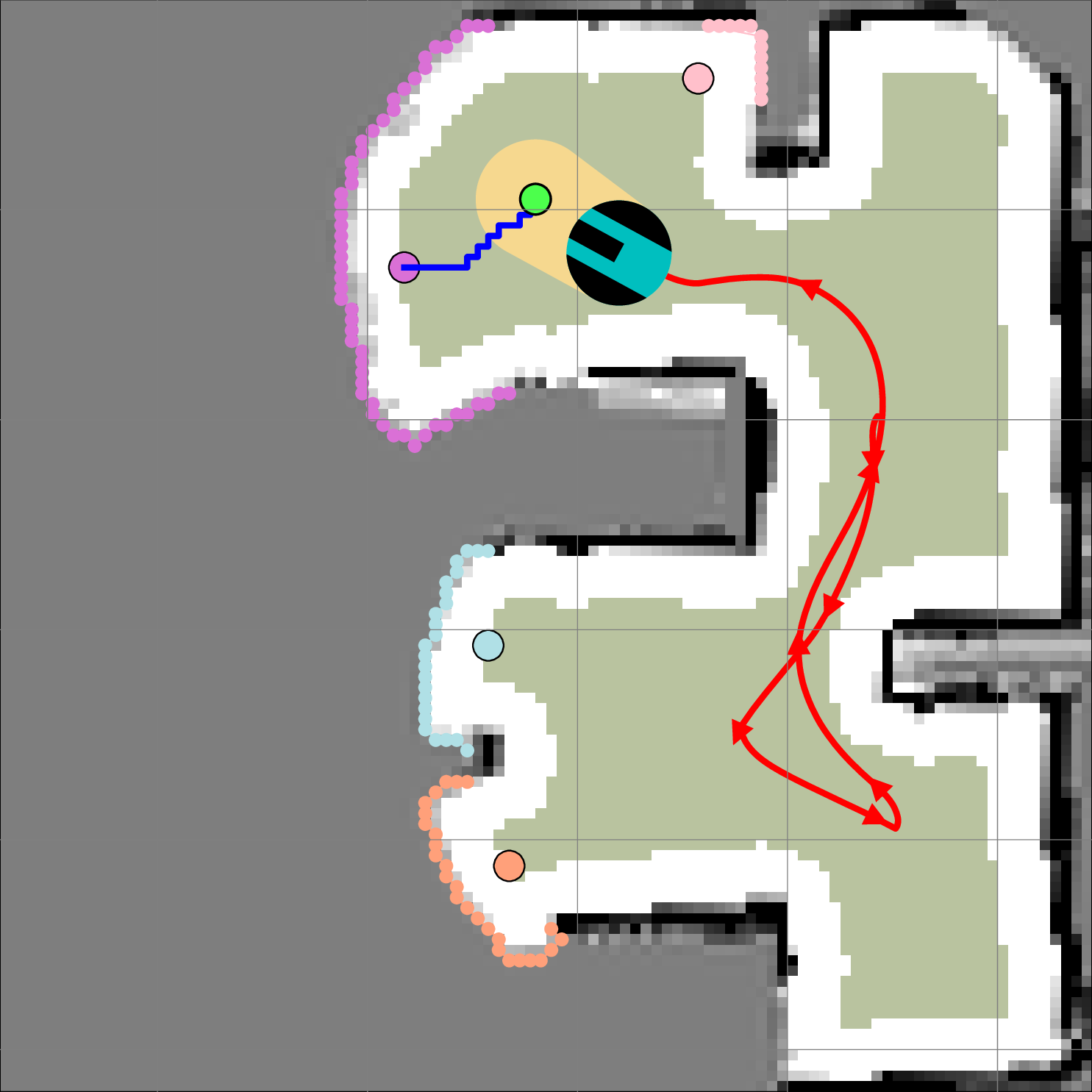} &
    \includegraphics[width = 0.195\textwidth]{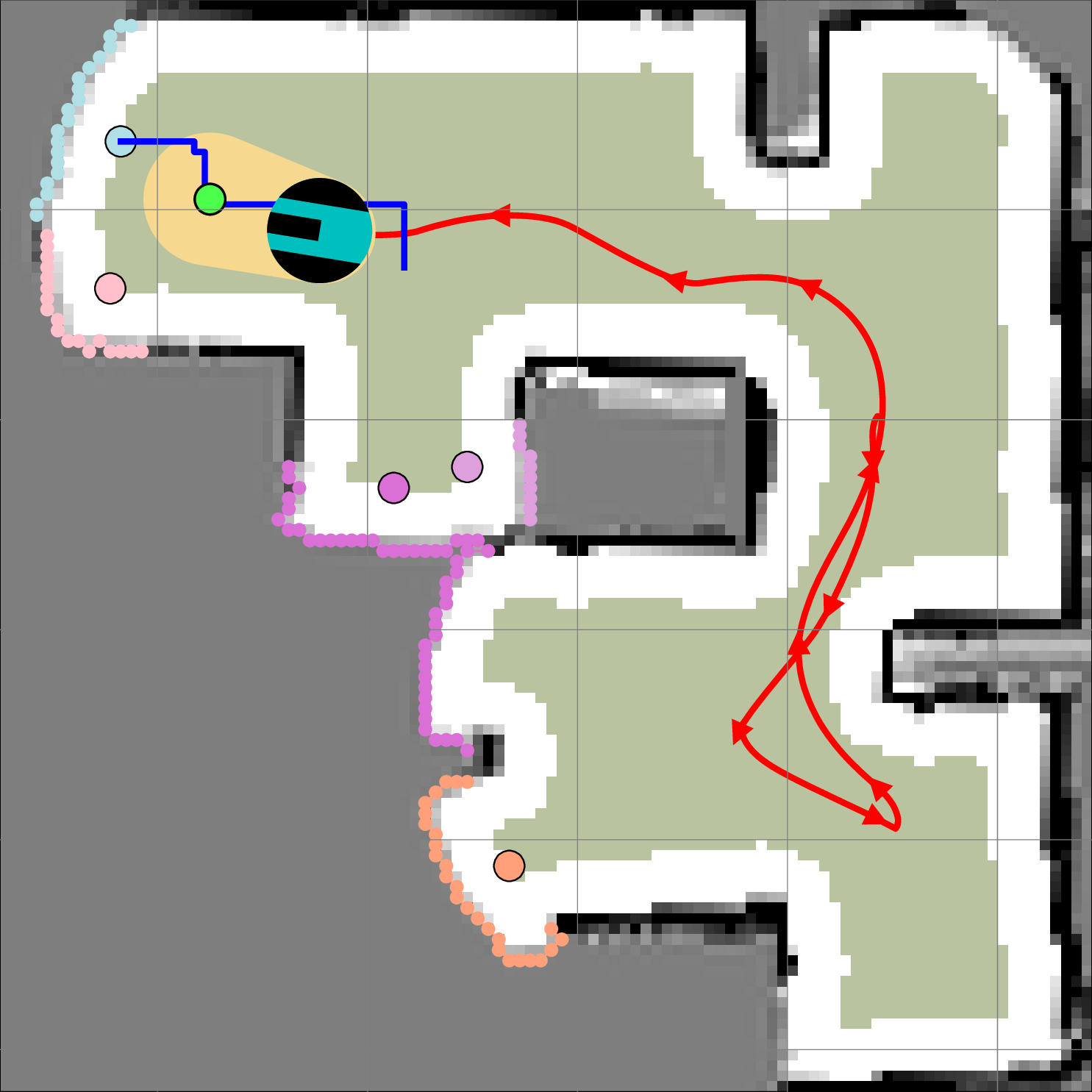} &
    \includegraphics[width = 0.195\textwidth]{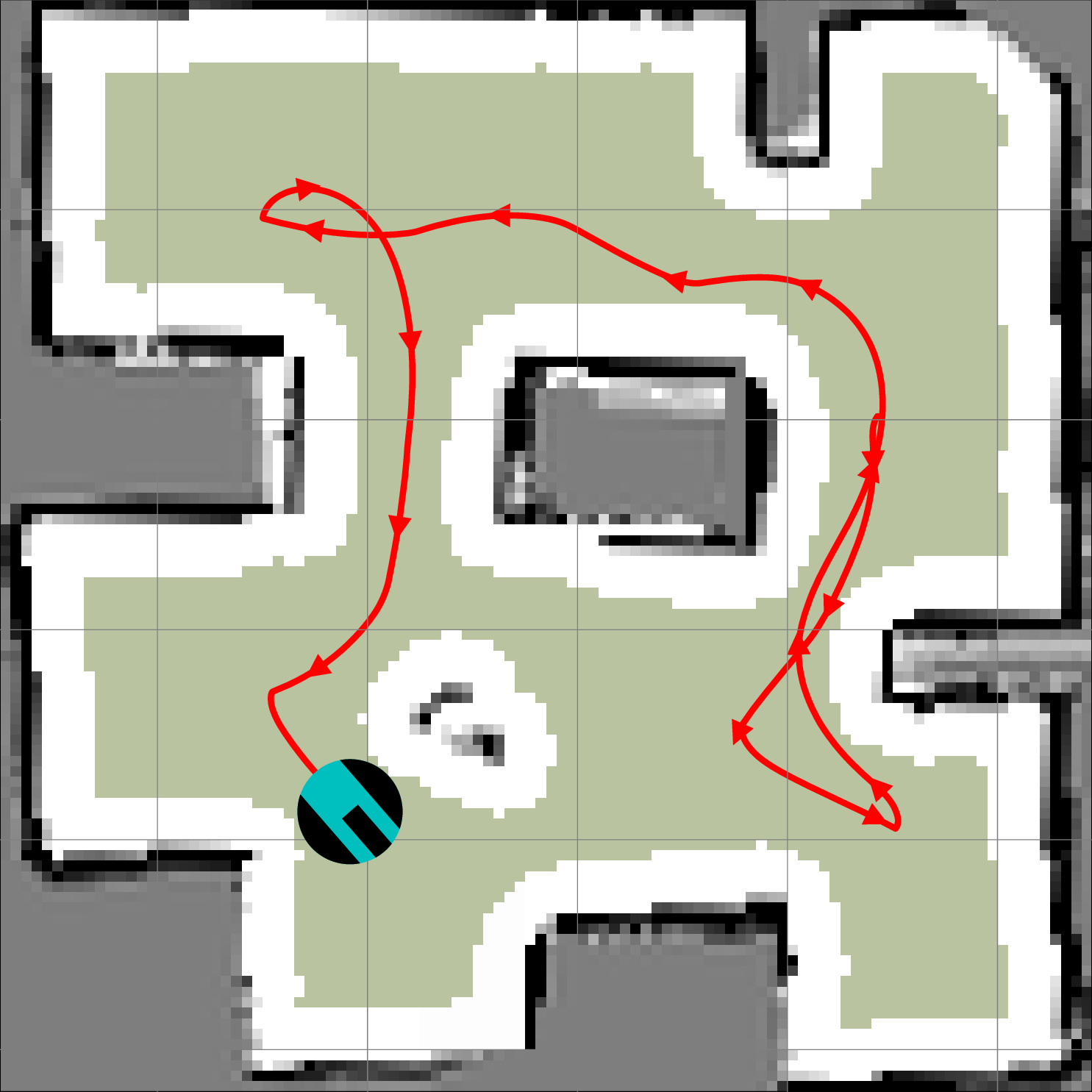}
    \\
    \includegraphics[width = 0.195\textwidth]{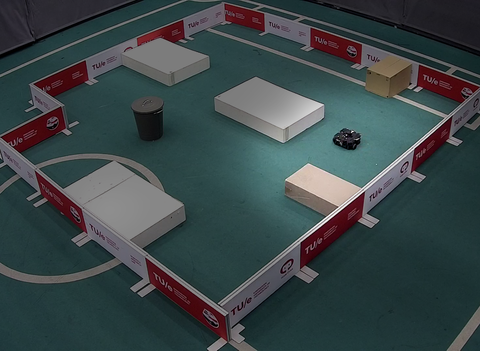} &
    \includegraphics[width = 0.195\textwidth]{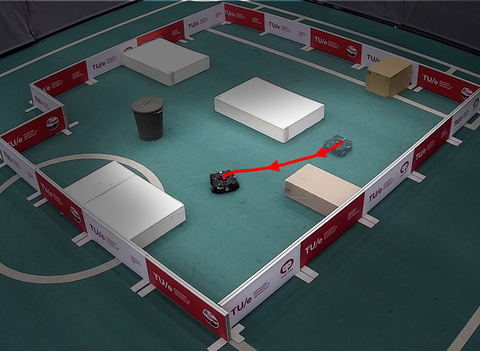} &
    \includegraphics[width = 0.195\textwidth]{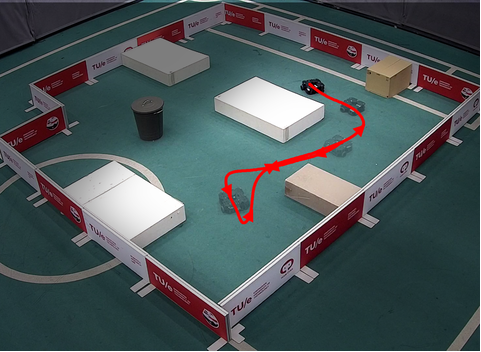} &
    \includegraphics[width = 0.195\textwidth]{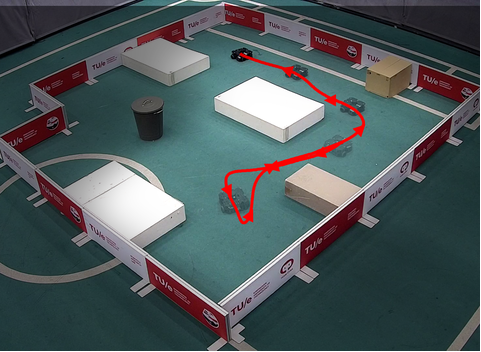} &
    \includegraphics[width = 0.195\textwidth]{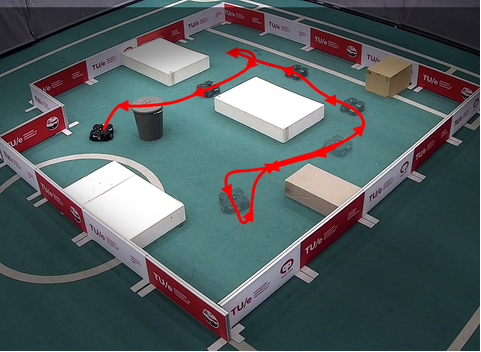}
    \\
    \includegraphics[width = 0.195\textwidth]{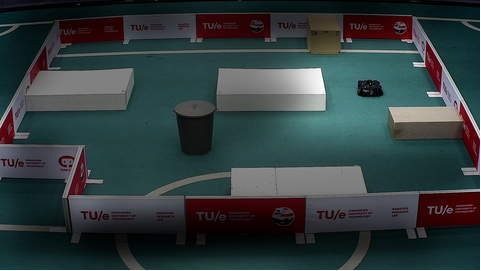} &
    \includegraphics[width = 0.195\textwidth]{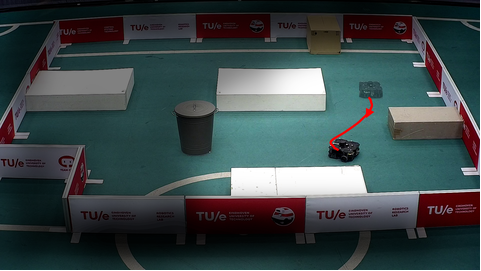} &
    \includegraphics[width = 0.195\textwidth]{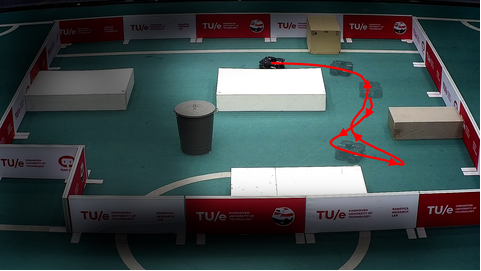} &
    \includegraphics[width = 0.195\textwidth]{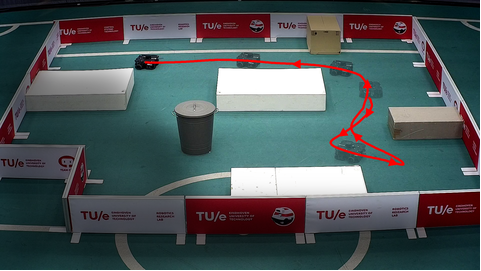} &
    \includegraphics[width = 0.195\textwidth]{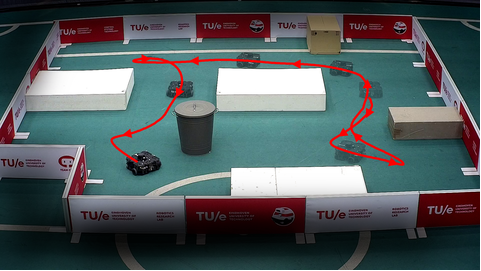}
    \\[-2mm]
    \end{tabular}
    \caption{Example snapshots of a TurtleBot3 Waffle Pi robot safely and effectively exploring and mapping a 6m$\times$6m cluttered lab environment using last-mile preventive replanning with geodesic navigation cost and frontier-size-based information utility. 
The robot's motion (red line) safely follows the reference path (blue line) toward the target viewpoint (colored circle) of a selected frontier region (colored lines), using feedback motion prediction (yellow patch) to reach the local control goal (green circle) on the path.
    }
    \label{fig.physical_experiments_sequential_exploration}
    \vspace{-2mm}
\end{figure*}

\begin{figure}[t]
    \centering
    \begin{tabular}{@{}c@{\hspace{0.005\columnwidth}}c@{\hspace{0.005\columnwidth}}c@{\hspace{0.005\columnwidth}}c@{}}
            \rotatebox{90}{\parbox{0.315\columnwidth}{\centering \scriptsize{Persistent Replanning} }} &
            \includegraphics[width = 0.315\columnwidth]{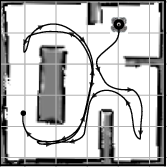} &
            \includegraphics[width = 0.315\columnwidth]{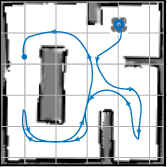} &
            \includegraphics[width = 0.315\columnwidth]{figures/expphysical_map1_withrobot_motionpatterns_persistent_infogeodesic.eps}    
        \\[-2mm]
            \rotatebox{0}{\centering \scriptsize{\phantom{.}} } &
            \rotatebox{0}{\parbox{0.315\columnwidth}{\centering \scriptsize{Purely Info-Driven} }} &
            \rotatebox{0}{\parbox{0.315\columnwidth}{\centering \scriptsize{Euclidean Cost} }} &
            \rotatebox{0}{\parbox{0.315\columnwidth}{\centering \scriptsize{Geodesic Cost} }}      
        \\[-2mm]
    \end{tabular}
    \caption{Example TurtleBot3 robot trajectories for exploration with persistent planning using frontier-size information utility and different navigation costs: (left) uniform, (middle) Euclidean, and (right) geodesic distance.}
    \vspace{-2mm}
    \label{fig.physical_experiments_persistent_planning}
    \end{figure}

\begin{figure}[t]
    \centering
    \begin{tabular}{@{}c@{\hspace{0.005\columnwidth}}c@{\hspace{0.005\columnwidth}}c@{\hspace{0.005\columnwidth}}c@{}}
            \rotatebox{90}{\parbox{0.315\columnwidth}{\centering \scriptsize{Geodesic Cost} }} &
            \includegraphics[width = 0.315\columnwidth]{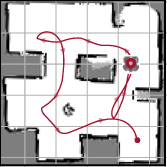} &
            \includegraphics[width = 0.315\columnwidth]{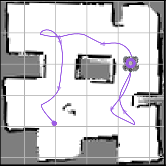} &
            \includegraphics[width = 0.315\columnwidth]{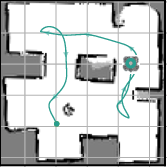}
        \\[-2mm]
            \rotatebox{0}{\centering \scriptsize{\phantom{.}} } &
            \rotatebox{0}{\parbox{0.315\columnwidth}{\centering \scriptsize{Persistent Planning} }} &
            \rotatebox{0}{\parbox{0.315\columnwidth}{\centering \scriptsize{Preventive Planning} }} &
            \rotatebox{0}{\parbox{0.315\columnwidth}{\centering \scriptsize{Online Planning} }}   
        \\[-2mm]
    \end{tabular}
    \caption{Example TurtleBot3 exploration trajectories for maximum frontier-size information utility per geodesic navigation cost with different planning strategies: (left) persistent, (middle) preventive, and (right) online planning.   
    }
    \vspace{-2mm}
    \label{fig.physical_experiments_geodesic_cost}
    \end{figure}

In this section, we consider the frontier-size-based information measure due to its positive influence on exploration and study the two other key factors, the navigation cost measure and the exploration planning strategy, on exploration performance, in a larger 28m$\times$19m office-like simulated environment, see \reffig{fig.initial_conditions}\,(right).
We present in \reffig{fig.effect_of_navigation_cost_and_replanning_strategy} the resulting robot trajectories for exploration with persistent planning, last-mile preventive planning, and online planning, using both Euclidean and geodesic navigation costs, as well as the mapping percentage over the traveled distance for each exploration strategy.
As expected, compared to persistent planning, exploration with last-mile preventive planning completes mapping faster with a shorter total travel distance, as the robot recognizes and avoids unnecessary exploration motion deep into the farthest corners of room-like regions, as seen in \reffig{fig.effect_of_navigation_cost_and_replanning_strategy}(middle).   
Exploration with online planning inherits and further extends the last-mile prevention property by leveraging the latest map information at the highest level of adaptability and results in less cluttered smoother robot trajectories, as seen in \reffig{fig.effect_of_navigation_cost_and_replanning_strategy} (right), but at the cost of higher computational cost due to periodic online replanning.
Also note that online exploration planning with the Euclidean distance may exhibit livelocks, where the robot continuously moves and switches between different frontier viewpoints without updating the map.
We observe that, regardless of the planning strategy, exploration with the geodesic distance performs better (or at least no worse) than with the Euclidean distance, since the geodesic navigation cost provides a more accurate estimate of the actual travel cost compared to the straight-line Euclidean distance.
Overall, compared to persistent and online planning, we conclude that exploration with last-mile preventive planning, using the frontier-size information measure and geodesic navigation cost in \reffig{fig.effect_of_navigation_cost_and_replanning_strategy} (middle, bottom), offers the best exploration performance at the lowest computational cost  without frequent replanning.

\subsection{Experimental Validation} 

For experimental validation of the proposed active exploration framework in practice, we conduct physical experiments using a differential-drive TurtleBot3 Waffle Pi platform equipped with a 2D $360^\circ$ laser scanner moving in  6m$\times$6m cluttered lab environments\footnote{The TurtleBot3 Waffle Pi robot platform has a body radius of 0.22m with respect to the motion center of its differential drive wheels and is equipped with a 2D $360^\circ$ LiDAR range scanner (LDS-01) generating 360 samples at 5Hz with a thresholded maximum sensing range of 1.5m.
The robot's pose is obtained from an OptiTrack motion capture system at 30Hz, and is controlled using the unicycle-path-following controller \refeq{eq.SafePathFollowing2} at 10Hz with path following coefficients of $\gain_\safelevel = 1$ and $\gain_\pathparam = 1$, robot control gains of $\lingain = 1$ and $\anggain = 2$, at a maximum linear velocity of 0.26m/s and a maximum angular velocity of 0.8rad/s.}, tracked by an OptiTrack motion capture system for localization, see \reffig{fig.physical_experiments_sequential_exploration}.
We first consider exploration with persistent planning using uniform, Euclidean, and geodesic navigation costs in \reffig{fig.physical_experiments_persistent_planning}.
Aligned with our observations in numerical simulations, the example robot trajectories in \reffig{fig.physical_experiments_persistent_planning} demonstrate that action-aware exploration reduces travel distance compared to purely information-driven exploration (with uniform navigation cost), and exploration with the geodesic distance completes mapping with a shorter total travel distance than with the Euclidean distance.
Because purely information-driven exploration switches between different viewpoints prematurely without fully exploring a region, and the Euclidean distance misguides the robot regarding the actual navigation cost to reach a viewpoint.
Secondly, we consider different exploration planning strategies using the frontier-size-based information measure and the geodesic navigation cost in \reffig{fig.physical_experiments_geodesic_cost}.
As in the numerical simulations, we observe a strong similarity between exploration with last-mile preventive and online planning, both of which significantly outperform persistent planning due to their adaptive nature.  
Finally, in \reffig{fig.physical_experiments_sequential_exploration}, we present snapshots of the TurtleBot3 robot incrementally mapping an unknown environment (from left to right) using the proactive exploration strategy with last-mile preventive planning, utilizing the frontier-size-based information measure and the geodesic navigation distance.
The robot effectively and safely explores and maps the environment while avoiding collisions and preventing deadlock and livelock situations.
In summary, the physical experiments confirm our numerical observations and validate the effectiveness of our proposed action-aware proactive exploration framework in practice.

\section{Conclusions}
\label{sec.conclusions}

In this paper, we present an action-aware proactive exploration framework for autonomous mobile robots to safely, reliably, and efficiently perform exploration and mapping in complex unknown environments.
We achieve this by tightly coupling and bridging the gap between perception and action in three layers of exploration: 
i) we optimally design safe and informative exploration paths, simultaneously minimizing the risk of collision and the distance to unexplored regions,
ii) we select safely reachable informative viewpoints with maximum information utility per total navigation cost,
iii) we preventively terminate and replan an exploration plan if a selected viewpoint does not provide enough actionable information.
Additionally, we ensure the safe execution of an exploration plan by systematically applying verifiably safe and correct path-following control, as well as identifying potential discrepancies between planning and control early in the design of maximal clearance safe reference paths in planning. 
In both numerical simulation and hardware experiments, we observe that action-aware proactive exploration with last-mile preventive planning offers the best balance of adaptivity, computational efficiency, and exploration efficiency, compared to persistent and online planning strategies.
We conclude that the cost of navigation is more influential in robotic exploration for occupancy mapping using dense and accurate (e.g., LiDAR) range sensors than the information utility, since all informative frontier regions, irrespective of their information content, eventually need to be observed to complete the mapping. 
Hence, this raises questions about the effectiveness of information-only exploration, especially in large environments.

Our current work focuses on multi-viewpoint planning for action-aware active exploration using approximate traveling salesman methods and the systematic combination of reactive and proactive planning in exploration.
Another promising research direction is action-aware exploration for simultaneous mapping and localization that minimizes the cost of action while reducing the uncertainty in mapping and localization, particularly in dynamic environments.
 

%
%

\bibliographystyle{IEEEtran}
\bibliography{references}

\begin{thebibliography}{10}
\providecommand{\url}[1]{#1}
\csname url@rmstyle\endcsname
\providecommand{\newblock}{\relax}
\providecommand{\bibinfo}[2]{#2}
\providecommand\BIBentrySTDinterwordspacing{\spaceskip=0pt\relax}
\providecommand\BIBentryALTinterwordstretchfactor{4}
\providecommand\BIBentryALTinterwordspacing{\spaceskip=\fontdimen2\font plus
\BIBentryALTinterwordstretchfactor\fontdimen3\font minus
  \fontdimen4\font\relax}
\providecommand\BIBforeignlanguage[2]{{%
\expandafter\ifx\csname l@#1\endcsname\relax
\typeout{** WARNING: IEEEtran.bst: No hyphenation pattern has been}%
\typeout{** loaded for the language `#1'. Using the pattern for}%
\typeout{** the default language instead.}%
\else
\language=\csname l@#1\endcsname
\fi
#2}}

\bibitem{stachniss_RoboticMappingExploration2009}
C.~Stachniss, \emph{Robotic mapping and exploration}.\hskip 1em plus 0.5em
  minus 0.4em\relax Springer Science \& Business Media, 2009.

\bibitem{calisi_etal_ISRR2005}
D.~Calisi, A.~Farinelli, L.~Iocchi, and D.~Nardi, ``Autonomous navigation and
  exploration in a rescue environment,'' in \emph{IEEE International Workshop
  on Safety, Security and Rescue Robotics}, 2005, pp. 54--59.

\bibitem{dipaola_IJARS2010}
D.~Di~Paola, A.~Milella, G.~Cicirelli, and A.~Distante, ``An autonomous mobile
  robotic system for surveillance of indoor environments,'' \emph{Int. Journal
  of Advanced Robotic Systems}, vol.~7, no.~1, p.~8, 2010.

\bibitem{ghaffari_etal_IJRR2019}
M.~Ghaffari~Jadidi, J.~Valls~Miro, and G.~Dissanayake, ``Sampling-based
  incremental information gathering with applications to robotic exploration
  and environmental monitoring,'' \emph{The International Journal of Robotics
  Research}, vol.~38, no.~6, pp. 658--685, 2019.

\bibitem{otsu_agah-mohammadi_paton_RAL2018}
K.~Otsu, A.-A. Agha-Mohammadi, and M.~Paton, ``Where to look? predictive
  perception with applications to planetary exploration,'' \emph{IEEE Robotics
  and Automation Letters}, vol.~3, no.~2, pp. 635--642, 2018.

\bibitem{yamauchi_CIRA1997}
B.~Yamauchi, ``A frontier-based approach for autonomous exploration,'' in
  \emph{IEEE International Symposium on Computational Intelligence in Robotics
  and Automation}, 1997, pp. 146--151.

\bibitem{placed_etal_TRO2023}
J.~A. Placed, J.~Strader, H.~Carrillo, N.~Atanasov, V.~Indelman, L.~Carlone,
  and J.~A. Castellanos, ``A survey on active simultaneous localization and
  mapping: State of the art and new frontiers,'' \emph{IEEE Transactions on
  Robotics}, vol.~39, no.~3, pp. 1686--1705, 2023.

\bibitem{lluvia_etal_Sensors2021}
I.~Lluvia, E.~Lazkano, and A.~Ansuategi, ``Active mapping and robot
  exploration: A survey,'' \emph{Sensors}, vol.~21, no.~7, p. 2445, 2021.

\bibitem{elfes_Computer1989}
A.~Elfes, ``Using occupancy grids for mobile robot perception and navigation,''
  \emph{Computer}, vol.~22, no.~6, pp. 46--57, 1989.

\bibitem{thrun_AR2003}
S.~Thrun, ``Learning occupancy grid maps with forward sensor models,''
  \emph{Autonomous Robots}, vol.~15, no.~2, pp. 111--127, 2003.

\bibitem{thrun_probabilistic2005}
S.~Thrun, W.~Burgard, and D.~Fox, \emph{Probabilistic Robotics}.\hskip 1em plus
  0.5em minus 0.4em\relax MIT Press, 2005.

\bibitem{keidar_kaminka_IJRR2014}
M.~Keidar and G.~A. Kaminka, ``Efficient frontier detection for robot
  exploration,'' \emph{The International Journal of Robotics Research},
  vol.~33, no.~2, pp. 215--236, 2014.

\bibitem{umari_mukhopadhyay_IROS2017}
H.~Umari and S.~Mukhopadhyay, ``Autonomous robotic exploration based on
  multiple rapidly-exploring randomized trees,'' in \emph{IEEE/RSJ Int. Conf.
  on Intelligent Robots and Systems}, 2017, pp. 1396--1402.

\bibitem{grisetti_etal_MITS2010}
G.~Grisetti, R.~Kümmerle, C.~Stachniss, and W.~Burgard, ``A tutorial on
  graph-based slam,'' \emph{IEEE Intelligent Transportation Systems Magazine},
  vol.~2, no.~4, pp. 31--43, 2010.

\bibitem{latha_arslan_arXiv2024}
D.~B. Latha and {\"O}.~Arslan, ``Key-scan-based mobile robot navigation:
  Integrated mapping, planning, and control using graphs of scan regions,''
  \emph{arXiv:2409.13838}, 2024.

\bibitem{gomez_etal_Sensors2019}
C.~Gomez, A.~C. Hernandez, and R.~Barber, ``Topological frontier-based
  exploration and map-building using semantic information,'' \emph{Sensors},
  vol.~19, no.~20, p. 4595, 2019.

\bibitem{kim_RAL2022}
H.~Kim, H.~Kim, S.~Lee, and H.~Lee, ``Autonomous exploration in a cluttered
  environment for a mobile robot with 2d-map segmentation and object
  detection,'' \emph{IEEE Robotics and Automation Letters}, vol.~7, no.~3, pp.
  6343--6350, 2022.

\bibitem{lu_redondo_campoy_Sensors2020}
L.~Lu, C.~Redondo, and P.~Campoy, ``Optimal frontier-based autonomous
  exploration in unconstructed environment using {RGB-D} sensor,''
  \emph{Sensors}, vol.~20, no.~22, p. 6507, 2020.

\bibitem{uslu_INISTA2015}
E.~Uslu, F.~{\c{C}}akmak, M.~Balc{\i}lar, A.~Ak{\i}nc{\i}, M.~F. Amasyal{\i},
  and S.~Yavuz, ``Implementation of frontier-based exploration algorithm for an
  autonomous robot,'' in \emph{International Symposium on Innovations in
  Intelligent SysTems and Applications (INISTA)}, 2015, pp. 1--7.

\bibitem{bourgault_etal_IROS2002}
F.~Bourgault, A.~Makarenko, S.~Williams, B.~Grocholsky, and H.~Durrant-Whyte,
  ``Information based adaptive robotic exploration,'' in \emph{IEEE/RSJ Int.
  Conf. on Intell. Robots and Syst.}, 2002, pp. 540--545.

\bibitem{stachniss_grisetti_burgard_RSS2005}
C.~Stachniss, G.~Grisetti, and W.~Burgard, ``Information gain-based exploration
  using rao-blackwellized particle filters.'' in \emph{Robotics: Science and
  systems}, 2005, pp. 65--72.

\bibitem{charrow_etal_RSS2015}
B.~Charrow, G.~Kahn, S.~Patil, S.~Liu, K.~Goldberg, P.~Abbeel, N.~Michael, and
  V.~Kumar, ``Information-theoretic planning with trajectory optimization for
  dense 3d mapping.'' in \emph{Robotics: Science and Systems}, 2015, pp. 3--12.

\bibitem{connolly_ICRA1985}
C.~Connolly, ``The determination of next best views,'' in \emph{IEEE
  International Conference on Robotics and Automation}, 1985, pp. 432--435.

\bibitem{gonzalez-banos_latombe_IJRR2002}
H.~H. González-Baños and J.-C. Latombe, ``Navigation strategies for exploring
  indoor environments,'' \emph{The International Journal of Robotics Research},
  vol.~21, no. 10-11, pp. 829--848, 2002.

\bibitem{shrestha_etal_ICRA2019}
R.~Shrestha, F.-P. Tian, W.~Feng, P.~Tan, and R.~Vaughan, ``Learned map
  prediction for enhanced mobile robot exploration,'' in \emph{International
  Conference on Robotics and Automation}, 2019, pp. 1197--1204.

\bibitem{saroya_best_hollinger_IROS2020}
M.~Saroya, G.~Best, and G.~A. Hollinger, ``Online exploration of tunnel
  networks leveraging topological cnn-based world predictions,'' in
  \emph{IEEE/RSJ Int. Conf. on Intell. Robots and Syst.}, 2020, pp. 6038--6045.

\bibitem{niroui_etal_RAL2019}
F.~Niroui, K.~Zhang, Z.~Kashino, and G.~Nejat, ``Deep reinforcement learning
  robot for search and rescue applications: Exploration in unknown cluttered
  environments,'' \emph{IEEE Robotics and Automation Letters}, vol.~4, no.~2,
  pp. 610--617, 2019.

\bibitem{bai_etal_IROS2016}
S.~Bai, J.~Wang, F.~Chen, and B.~Englot, ``Information-theoretic exploration
  with bayesian optimization,'' in \emph{IEEE/RSJ International Conference on
  Intelligent Robots and Systems}, 2016, pp. 1816--1822.

\bibitem{julia_gil_reinoso_AR2012}
M.~Juli{\'a}, A.~Gil, and O.~Reinoso, ``A comparison of path planning
  strategies for autonomous exploration and mapping of unknown environments,''
  \emph{Autonomous Robots}, vol.~33, pp. 427--444, 2012.

\bibitem{zhou_etal_RAL2021}
B.~Zhou, Y.~Zhang, X.~Chen, and S.~Shen, ``Fuel: Fast uav exploration using
  incremental frontier structure and hierarchical planning,'' \emph{IEEE
  Robotics and Automation Letters}, vol.~6, no.~2, pp. 779--786, 2021.

\bibitem{mei_lu_lee_hu_ICRA2006}
Y.~Mei, Y.-H. Lu, C.~Lee, and Y.~Hu, ``Energy-efficient mobile robot
  exploration,'' in \emph{IEEE International Conference on Robotics and
  Automation}, 2006, pp. 505--511.

\bibitem{cieslewski_etal_IROS2017}
T.~Cieslewski, E.~Kaufmann, and D.~Scaramuzza, ``Rapid exploration with
  multi-rotors: A frontier selection method for high speed flight,'' in
  \emph{IEEE/RSJ Int. Conf. on Intell. Robots and Syst.}, 2017, pp. 2135--2142.

\bibitem{dai_etal_ICRA2020}
A.~Dai, S.~Papatheodorou, N.~Funk, D.~Tzoumanikas, and S.~Leutenegger, ``Fast
  frontier-based information-driven autonomous exploration with an {MAV},'' in
  \emph{IEEE International Conference on Robotics and Automation}, 2020, pp.
  9570--9576.

\bibitem{placed_IFAC2022}
J.~A. Placed and J.~A. Castellanos, ``Enough is enough: Towards autonomous
  uncertainty-driven stopping criteria,'' \emph{IFAC-PapersOnLine}, vol.~55,
  no.~14, pp. 126--132, 2022.

\bibitem{gao_etal_ICARCV2018}
W.~Gao, M.~Booker, A.~Adiwahono, M.~Yuan, J.~Wang, and Y.~W. Yun, ``An improved
  frontier-based approach for autonomous exploration,'' in \emph{Int. Conf. on
  Control, Automation, Robotics and Vision}, 2018, pp. 292--297.

\bibitem{strom_RAS2017}
D.~P. Str{\"o}m, I.~Bogoslavskyi, and C.~Stachniss, ``Robust exploration and
  homing for autonomous robots,'' \emph{Robotics and Autonomous Systems},
  vol.~90, pp. 125--135, 2017.

\bibitem{zheng_Spring2021}
K.~Zheng, ``{ROS} navigation tuning guide,'' \emph{Robot Operating System (ROS)
  The Complete Reference (Volume 6)}, pp. 197--226, 2021.

\bibitem{koenig_TRO2005}
S.~Koenig and M.~Likhachev, ``Fast replanning for navigation in unknown
  terrain,'' \emph{IEEE Transactions on Robotics}, vol.~21, no.~3, pp.
  354--363, 2005.

\bibitem{tordesillas_TRO2021}
J.~Tordesillas, B.~T. Lopez, M.~Everett, and J.~P. How, ``Faster: Fast and safe
  trajectory planner for navigation in unknown environments,'' \emph{IEEE
  Transactions on Robotics}, vol.~38, no.~2, pp. 922--938, 2022.

\bibitem{bircher_ICRA2016}
A.~Bircher, M.~Kamel, K.~Alexis, H.~Oleynikova, and R.~Siegwart, ``Receding
  horizon "next-best-view" planner for 3d exploration,'' in \emph{IEEE Int.
  Conf. on Robotics and Automation}, 2016, pp. 1462--1468.

\bibitem{dharmadhikari_etal_ICRA2020}
M.~Dharmadhikari, T.~Dang, L.~Solanka, J.~Loje, H.~Nguyen, N.~Khedekar, and
  K.~Alexis, ``Motion primitives-based path planning for fast and agile
  exploration using aerial robots,'' in \emph{IEEE International Conference on
  Robotics and Automation}, 2020, pp. 179--185.

\bibitem{lindqvist_etal_TRO2024}
B.~Lindqvist, A.~Patel, K.~L{\"o}fgren, and G.~Nikolakopoulos, ``A tree-based
  next-best-trajectory method for {3D UAV} exploration,'' \emph{IEEE
  Transactions on Robotics}, vol.~40, pp. 3496--3513, 2024.

\bibitem{hollinger_sukhatme_IJRR2014}
G.~A. Hollinger and G.~S. Sukhatme, ``Sampling-based robotic information
  gathering algorithms,'' \emph{The International Journal of Robotics
  Research}, vol.~33, no.~9, pp. 1271--1287, 2014.

\bibitem{arslan_arXiv2022}
{\"O}.~Arslan, ``Time governors for safe path-following control,'' \emph{arXiv
  preprint arXiv:2212.01444}, 2022.

\bibitem{isleyen_vandewouw_arslan_IROS2023}
A.~{\.I}{\c{s}}leyen, N.~van~de Wouw, and {\"O}.~Arslan, ``Feedback motion
  prediction for safe unicycle robot navigation,'' in \emph{IEEE/RSJ Int. Conf.
  on Intelligent Robots and Systems}, 2023, pp. 10\,511--10\,518.

\bibitem{isleyen_vandewouw_arslan_cdc2023}
A.~{\.I}{\c{s}}leyen, N.~Van De~Wouw, and {\"O}.~Arslan, ``Adaptive headway
  motion control and motion prediction for safe unicycle motion design,'' in
  \emph{IEEE Conference on Decision and Control}, 2023, pp. 6942--6949.

\bibitem{holz_etal_ISR2010}
D.~Holz, N.~Basilico, F.~Amigoni, and S.~Behnke, ``Evaluating the efficiency of
  frontier-based exploration strategies,'' in \emph{International Symposium on
  Robotics}, 2010, pp. 1--8.

\end{thebibliography}


\end{document}